\documentclass[11pt]{article}
\usepackage[margin=1.1in,a4paper]{geometry}

\usepackage[utf8]{inputenc} 
\usepackage[T1]{fontenc}    
\usepackage{hyperref}       
\usepackage{url}            
\usepackage{booktabs}       
\usepackage{amsfonts}       
\usepackage{nicefrac}       
\usepackage{microtype}      

\usepackage{amsmath}

\RequirePackage{algorithm}
\RequirePackage{algorithmic}

\usepackage{graphicx}
\usepackage{multirow}
\usepackage{pgf,tikz}
\usepackage{pgfplots}
\pgfplotsset{compat=1.13}
\usetikzlibrary{spy,shapes.misc}
\usepackage{wrapfig}
\graphicspath{{figs/}}
\usepackage{grffile}

\usepackage{hyperref}
\hypersetup{colorlinks=True, linkcolor=blue, citecolor=blue}

\usepackage{caption}
\usepackage{enumitem}
\setlist[itemize]{leftmargin=5mm}

\def\eqref#1{Eq.~(\ref{#1})}

\def\1{\bm{1}}

\DeclareMathAlphabet{\mathsfit}{\encodingdefault}{\sfdefault}{m}{sl}
\SetMathAlphabet{\mathsfit}{bold}{\encodingdefault}{\sfdefault}{bx}{n}

\newcommand{\R}{\mathbb{R}}

\DeclareMathOperator*{\argmin}{arg\,min}

\DeclareMathOperator{\Tr}{Tr}

\newcommand{\ud}{\,\mathrm{d}}
\def\R{{\mathbb R}}
\newcommand{\eqdef}{\ensuremath{\stackrel{\mbox{\upshape\tiny def.}}{=}}}
\let\on=\operatorname

\usepackage{amsthm}
\usepackage{theoremref}
\numberwithin{equation}{section}
\theoremstyle{plain}
\newtheorem{theorem}{Theorem}
\newtheorem{lemma}[theorem]{Lemma}
\newtheorem{corollary}[theorem]{Corollary}
\newtheorem{proposition}[theorem]{Proposition}

\theoremstyle{definition}

\newcommand{\x}{\mathbf{x}}
\renewcommand{\v}{\mathbf{v}}
\newcommand{\z}{\mathbf{z}}
\newcommand{\q}{\mathbf{q}}
\newcommand{\p}{\mathbf{p}}
\newcommand{\UpDown}{\texttt{UpDown}}
\newcommand{\ResNet}{\texttt{ResNet}}
\newcommand{\ResNets}{\texttt{ResNet}s}
\newcommand{\PyTorch}{\texttt{PyTorch}}
\newcommand{\Adam}{\texttt{Adam}}
\newcommand{\ReLU}{\texttt{ReLU}}
\newcommand{\inflation}{\alpha}
\newcommand{\samplesize}{n}
\pagecolor{white}

\definecolor{b1}{HTML}{1d3557}
\definecolor{r1}{HTML}{e76f51}
\definecolor{m1}{HTML}{2a9d8f}
\definecolor{mygreen}{rgb}{0.0,0.3020,0.2196}
\definecolor{backprop}{rgb}{0.8235,0.2902,0.3098}

\newcommand{\mnrev}[1]{{\color{black}{#1}}}

\usepackage[disable]{todonotes} 
\usepackage{environ}
\newcommand{\acksection}{\section*{Acknowledgments and Disclosure of Funding}}
\NewEnviron{ack}{%
  \acksection
  \BODY
}

\author{Fran\c{c}ois-Xavier Vialard\thanks{LIGM, Univ Gustave Eiffel, CNRS, \texttt{francois-xavier.vialard@u-pem.fr}}\;
  \and Roland Kwitt\thanks{
  Department of Computer Science
  University of Salzburg;
  \texttt{Roland.Kwitt@sbg.ac.at}}\;
  \and
  Susan Wei\thanks{
  School of Mathematics and Statistics
  University of Melbourne;
  \texttt{susan.wei@unimelb.edu.au}}\;
  \and
  Marc Niethammer\thanks{
  Department of Computer Science
  University of North Carolina at Chapel Hill; 
  \texttt{mn@cs.unc.edu}}
}

\title{A Shooting Formulation of Deep Learning}
\begin{document}

\maketitle

\begin{abstract}
A residual network may be regarded as a discretization of an ordinary differential equation (ODE) which, in the limit of time discretization, defines a continuous-depth network. Although important steps have been taken to realize the advantages of such continuous formulations, most current techniques assume \textit{identical} layers. Indeed, existing works throw into relief the myriad difficulties of learning an infinite-dimensional parameter in a continuous-depth neural network. To this end, we introduce a shooting formulation which shifts the perspective from parameterizing a network layer-by-layer to parameterizing over \textit{optimal} networks described only by a set of \textit{initial conditions}. For scalability, we propose a novel particle-ensemble parameterization which fully specifies the optimal weight trajectory of the continuous-depth neural network. Our experiments show that our particle-ensemble shooting formulation can achieve competitive performance. 
Finally, though the current work is inspired by continuous-depth neural networks, the particle-ensemble shooting formulation also applies to discrete-time networks and may lead to a new fertile area of research in deep learning parameterization.
\end{abstract}

\section{Introduction}

Deep neural networks (DNNs) are closely related to optimal control (OC) where the sought-for control variable corresponds to the parameters of the DNN ~\cite{liu2019deep,li2017maximum,haber2017stable}. To be able to talk about an \emph{optimal} control requires the definition of a control cost, i.e., a norm on the control variable. We explore the ramifications of such a control cost in the context of DNN parameterization. For simplicity, we focus on continuous formulations in the spirit of neural ODEs~\cite{chen2018neural}. However, both discrete and continuous OC formulations exist~\cite{brysonapplied,athans2013optimal,troutman2012variational}; our approach could be developed for both. 

\begin{figure}[h!]
\centering{
\includegraphics[width=\textwidth]{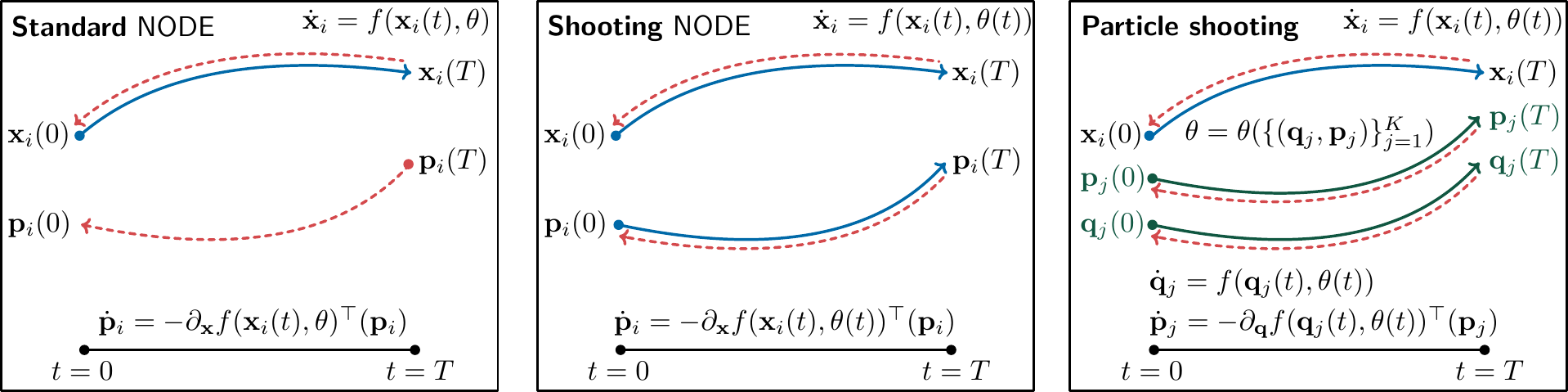}}
\caption{Optimization in the neural ODE (NODE) framework \cite{chen2018neural} (left) amounts to a forward pass with the gradient computed via \textcolor{backprop}{backpropagation} (\protect\tikz[baseline]{\protect\draw[line width=0.3mm,densely dashed,color=backprop] (0,.6ex)--++(0.3,0) ;}). Optimization under the shooting principle (middle) turns the forward-backward system into a forward \textit{second-order} system, where we essentially run the backpropagation equation forward.  We use a Hamiltonian particle ensemble (right) consisting of $K$ (position, momentum) pairs $(\q_j, \p_j)$ to make shooting efficient. Note that we write $\theta = \theta(\{(\q_j,\p_j) \}_{j=1}^K)$ since $\theta$ satisfies a compatibility equation which involves all $K$ particles. In shooting $\theta$ is time-dependent, in standard NODE $\theta(t)=\theta~\forall t$.
\label{fig:shooting_principle}}
\vspace{-0.4cm}
\end{figure}

Initial work on continuous DNN formulations was motivated by the realization that a \ResNet\ \cite{he2016deep,he2016identity} resembles Euler forward time-integration~\cite{haber2017stable,li2017maximum}. Specifically, the forward pass of some input vector $\tilde{\x} \in \mathbb R^d$ through a network with $L$ layers, specified as $\x(0) = \tilde{\x}$ and
$
\x(j+1) = \x(j) + f(\x(j),\theta(j)), \, j=0,1,\ldots,L, 
$
closely relates to an explicit Euler~\cite{trefethen1996finite} discretization of the ODE
\begin{equation}
\dot \x(t) = f(t,\x(t),\theta(t)), \quad  \x(0)=\tilde{\x}, \quad 0\le t \le T\enspace.
\label{ode_constraint}
\end{equation}
In the continuous DNN formulation, we seek an optimal $\theta$ such that the terminal prediction given by $\x(T)$, i.e., the solution to \eqref{ode_constraint} at time $T$, minimizes $\ell(\x(T))$ for a task-specific loss function $\ell$.

Although \eqref{ode_constraint} with time-varying parameter $\theta(t)$ can be considered as a neural network with an infinite number of layers, current implementations of ODE-inspired networks largely assume the parameters $\theta$ are fixed in time, i.e., $\forall t: \theta(t) = \theta$ ~\cite{chen2018neural,dupont2019augmented}, or follow some prescribed dynamics~\cite{zhang2019anodev2}. Instead, we explore time-varying $\theta(t)$ by employing regularization (i.e., a control cost) to render the estimation well-posed and to assure regularity of the resulting flow. Specifically, (for a single data point) we propose minimizing over $\theta$ the regularized loss
\begin{equation}
\mathcal E(\theta ) = \int_0^T R(\theta(t)) \,\ud t + \gamma~ \ell(\x(T)), \quad \gamma \in \mathbb R^+,\quad \text{subject to \eqref{ode_constraint}}\enspace,
\label{complete_energy}
\end{equation}
where $R$ is a real-valued complexity measure of $\theta$ corresponding to the control cost. We will mostly work with the Frobenius norm but $R(\theta(t))$ can be more general (see Appendix \ref{appendix:Barron}). 

Instead of \emph{directly} optimizing over the set of time-dependent $\theta(t)$ as in standard \ResNets, {\it we restrict the optimization set to those  $\theta$ which are critical points of $\mathcal E(\theta)$, thereby dramatically reducing the number of parameters.}
In doing so, one can describe the optimization task as an \textit{initial value problem}. Namely, we show that
we can rewrite the loss in~\eqref{complete_energy} solely in terms of the input $\x(0)$ and a corresponding finite-dimensional momentum variable, $\mathbf{p}(0)$. Such an approach, just like optimizing the initial speed of a mass particle to reach a given point, is called a \textit{shooting method} in numerical analysis~\cite{press2007numerical} and control~\cite{bonnans2013shooting}, giving its name to our new formulation. 

The first two panels of Fig.~\ref{fig:shooting_principle} illustrate the difference between the optimization of a neural ODE (NODE) via~\cite{chen2018neural} and our shooting formulation. Since in practice, we have multiple inputs $\tilde{\x}_i, i=1,\ldots,n$, there is an initial momentum vector $\mathbf{p}_i$ corresponding to each of them. If the shooting formulation is to scale up to a large sample size $n$, we must take care that the parameterization does not grow linearly with $n$. To this end, we propose what we call the \emph{Hamiltonian particle-ensemble parameterization}. It is a finite set of particles, where each particle is a (position, momentum) pair. The initial conditions of these particle pairs $\{(\q_j,\p_j)\}_{j=1}^K$ (where $K\ll n$) completely determine $\theta(t)$. This is illustrated in the rightmost panel of Fig.~\ref{fig:shooting_principle}. Once the optimized set of particles has been computed, the computational efficiency of the forward model, similarly to NODE~\cite{chen2018neural}, is retained for vector fields $f$ that are linear in their parameters $\theta(t)$.

Our \textbf{contributions} are as follows:
1) We introduce a shooting formulation for DNNs, amounting to an initial-value formulation for neural network parameterization. This allows for optimization over the original network parameter space via optimizing over the initial conditions of critical networks only;
2) We propose an efficient implementation of the shooting approach based on a novel particle-ensemble parameterization in which a set of initial particles (the (position, momentum) pairs) describe the space of putative optimal network parameters; 
3) We propose the \UpDown\ model which gives rise to explicit shooting equations;
4) We prove universality for the flows of the \UpDown\ vector field and demonstrate in experiments its good performance on several prediction tasks.

\vspace{-0.2cm}
\section{Related work}
We draw inspiration from two separate branches of research: 1) continuous formulations of neural networks~\cite{chen2018neural} and 2) shooting approaches for deformable image registration~\cite{Vialard2012,Miller2006,Niethammer2011}.

\textbf{Continuous-depth neural networks.} 
Continuous equivalents of \ResNets\ ~\cite{he2016deep,he2016identity} have been developed in~\cite{ruthotto2019deep,haber2017stable}, but na\"ive implementations are memory-demanding since backpropagation requires differentiating through the numerical integrator. Two approaches can address this unfavorable memory footprint. NODE~\cite{chen2018neural} does not store intermediate values in the forward pass, but recomputes them by integrating the forward model backward. This is easily possible only if the forward model is numerically invertible and the formulation is time-continuous~\cite{gholami2019anode}\footnote{In a discrete setting, resolving the forward model in the backward direction generally requires costly solving of implicit equations. This can be done (it is, e.g., done for invertible \ResNets\ ~\cite{behrmann2018invertible}). In general, an explicit numerical solution for forward time-integration becomes implicit in the backward direction and vice versa.}. Instead, checkpointing ~\cite{gholami2019anode} is a general approach to reduce memory requirements by selectively recomputing parts of the forward solution~\cite{griewank2008evaluating}. Our work can easily be combined with these numerical approaches.

\textbf{Solving implicit equations.}
A recent line of works, including deep equilibrium models~\cite{bai2019deep} and implicit residual networks~\cite{reshniak2019robust}, has shown that it may not always be necessary to freely parameterize all the layers in the network. Specifically, in~\cite{bai2019deep} and~\cite{reshniak2019robust}, the parameters of each layer are defined via an implicit equation motivated by \emph{weight tying} thus improving expressiveness and reducing the number of parameters while decreasing the memory footprint via implicit differentiation. Instead, our work increases expressiveness and reduces the number of parameters via particle-based shooting.

\textbf{Invertibility and expressiveness.}
Based on similarity with continuous time integration, constraining the norm of a layer in a \ResNet\ will result in an invertible network such as in ~\cite{behrmann2018invertible,jacobsen2018revnet}. Invertibility is also explored in~\cite{younes2018diffeomorphic}, where it is enforced (as in our setting) via a penalty of the norm. These works show that standard learning tasks can be performed on top of a one-to-one transformation. Recent theoretical developments~\cite{zhang_approximation_2020} show that indeed capping a NODE or i-ResNet~\cite{behrmann2018invertible} with a single linear layer gives universal approximation for non-invertible continuous functions. Further, expressiveness can be increased by moving to more complex models, e.g., by introducing additional dimensions as explored in augmented NODE~\cite{dupont2019augmented}. In \cite{zhang2019anodev2} (\texttt{AnodeV2}), Zhang et al. treat time-dependent $\theta(t)$. Weights are evolved jointly with the state of the continuous DNN. While this weight evolution could, in principle, also be captured by a learned weight network, the authors argue that this would result in a large increase in parameters and therefore opt for explicitly parameterizing these evolutions (e.g., via a reaction diffusion equation). In contrast, our method does not rely on learning a separate weight-network or on explicitly specifying a weight evolution. Instead, our evolving weights are a direct consequence of the shooting equations which, in turn, are a direct consequence of penalizing network parameters (the control cost) over time; a large increase in parameters does not occur.

\textbf{Hamiltonian approaches.}
Toth et al.~\cite{toth2019hamiltonian} proposed Hamiltonian generative networks to learn the Hamiltonian governing the evolution of a physical system. Specifically, they learn Hamiltonian vector fields in the latent space of an image encoder-decoder architecture. S{\ae}mundsson et al.~\cite{saemundsson2019variational} also learn the underlying dynamics of a system from time-dependent data, starting from a discrete Lagrangian combined with a variational integrator. This motivates particular network structures; e.g., Newtonian networks where the potential energy is learned via a neural network. Although sharing common tools, our work completely differs from this line of research in the sense that we exploit Hamiltonian mechanics to parameterize \emph{general} continuous neural networks. In principle, our work applies to most network architectures and is not specific to physical data.

Finally, we mention that shooting approaches have been applied successfully in other areas such as diffeomorphic image matching \cite{Miller2006,Vialard2012,Niethammer2011}. However, the decisive difference here is in the dimensionality of the underlying space: in diffeomorphic image registration, the data are points in a 3D volume i.e., $d=3$; for DNNs applications, data points usually lie in a much higher-dimensional space, i.e., $d$ is very large. 



\section{Shooting formulation of ODE-inspired neural networks}
\label{sec:general_framework}

We consider, for simplicity, a supervised learning task where the input and target spaces are $\mathcal X \subset \mathbb R^d$ and $\mathcal Y$, resp., and sampled data are denoted 
by $\{(\tilde{\x}_i,\tilde{\mathbf{y}}_i)\}_{i=1}^\samplesize \subset \mathcal X \times \mathcal Y$. The goal is to learn the weight $\theta(t)$ in the following flow equation 
\begin{equation}\label{EqFlow}
\dot{\x}_i(t) = f(\x_i(t),\theta(t)), \quad \x_i(0) = \tilde{\x}_i, \quad 0\le t \le T, \quad i=1,\ldots,\samplesize
\end{equation}
such that it minimizes the loss $\sum_{i = 1}^\samplesize \ell(\x_i(T),\tilde{\mathbf{y}}_i)$ for some loss function $\ell$. In existing works, the weight is chosen independent of time, i.e., $\theta(t) = \theta$ \cite{chen2018neural}, or specific evolution equations are postulated for it~\cite{massaroli_dissecting_2020,zhang2019anodev2}. Such strategies show the difficulty of addressing infinite dimensional parameterizations of time-dependent $\theta$ and the need for regularization for well-posedness~\cite{finlay_how_2020,massaroli_dissecting_2020,haber2017stable}. Instead of parameterizing $\theta(t)$ directly, we aim at penalizing $\theta(t)$ according to the regularity of $f(\cdot,\theta(t))$ to arrive at a well-posed problem. Specifically, we consider a regularization term $R(\theta(t))$ (discussed in \S\ref{SecChoiceOfReg}) and propose to minimize over $\theta$
\begin{equation}\label{EqVariationalFormulation}
\mathcal E_n(\theta) = \int_0^T R(\theta(t)) \,\ud t + \gamma \sum_{i=1}^\samplesize \ell(\x_i(T),\tilde{\mathbf{y}}_i), \quad \gamma \in \mathbb R^+, \quad \text{subject to \eqref{EqFlow}}\enspace.
\end{equation}
Note that upon discretizing the time $t$ (i.e., having a number of parameters proportional to the number of timesteps) this is similar to a \ResNet\ with weight decay.
For a \ResNet\ or a NODE, optimization is based on computing the parameter gradient via a forward pass followed by backpropagation (see left panel of Fig.~\ref{fig:shooting_principle}).

\textbf{Optimality equations.}
The optimality conditions for~\eqref{EqVariationalFormulation} in continuous time are:
\begin{align}
\begin{cases}
\dot{\x}_i(t) - f(\x_i(t),\theta(t)) = 0,~\x_i(0) = \tilde{\x}_i, & \text{Data evolution} \\
\dot{\p}_i(t) + \partial_\x f(\x_i(t),\theta(t))^\top(\p_i) = 0,~\p_i(T) = -\gamma \nabla \ell(\x_i(T),\tilde{\mathbf{y}}_i),& \text{Adjoint evolution}\\
\partial_\theta R(\theta(t)) - \sum_{i = 1}^\samplesize\partial_\theta f(\x_i(t),\theta(t))^\top(\p_i(t)) = 0 
 \,. & \text{Compatibility}
 \end{cases}
\label{EqOptimalityEquations}
\end{align}

The first equation describes evolution of the input data and the second equation is the adjoint equation solved backward in time in order to compute the gradient with respect to the parameters. At convergence, the third equation is also satisfied. This last equation encodes the optimality of the layer at timestep $t$, as it is the case for an \emph{argmin layer} or \emph{weight tying}~\cite{br2017optnet}. Its left hand side corresponds to the gradient with respect to the parameter $\theta$, but as we shall see it will allow us to compute $\theta$ directly via our (position, momentum) pairs in our particle shooting formulation. The shooting approach simply replaces the optimization set by the set of critical points of \eqref{EqVariationalFormulation} expressed in these optimality conditions. That is, we only optimize over solutions fulfilling~\eqref{EqOptimalityEquations}.

\textbf{Shooting principle.} The shooting method is standard in optimal control~\cite{bonnans2013shooting} and can be formulated as follows: since, at optimality, the system in \eqref{EqOptimalityEquations} is satisfied, one can \emph{turn this system into a forward model} defined only by its initial conditions $\{(\x_i(0),\p_i(0))\}_{i=1}^\samplesize$ which specify the \emph{entire trajectory} of optimal parameters. \mnrev{We evolve \emph{both} the data and adjoint evolution equations forward in time and compute at each time, $t$, $\theta(t)$ from the compatibility \eqref{EqOptimalityEquations} via the current values of $\{(\x_i(t),\p_i(t))\}_{i=1}^\samplesize$.} We refer to the forward model defined by \eqref{EqOptimalityEquations} as the \textit{shooting equations}.
Unfortunately, this initial-condition parameterization still requires \emph{all} initial conditions $\x_i(0)$ and their corresponding momenta $\p_i(0)$ for $i=1,\ldots,\samplesize$. Since this does not scale to very large datasets, we propose an approximation using a collection of particles, as described next.

\textbf{Hamiltonian particle ensemble.}
In the limit and ideal case where the data distribution is known, the optimality equations can be approximated using a collection of particles which follow the Hamiltonian system (see Appendix \ref{appendix:SecExpectationApproximation}). 
We thus consider a collection of particles $\{(\q_j,\p_j)\}_{j = 1}^K \in \R^d \times \R^d$ that drive the evolution of the entire population $\{\x_i\}_{i = 1}^\samplesize \subset \R^d$ through the following forward model
\begin{align}\label{EqShootingForward}
\begin{cases}
\,\dot{\x}_i(t) - f(\x_i(t),\theta(t)) = 0,~\x_i(0)=\tilde{\x}_i & \text{Data evolution}\\
 \left.\begin{aligned}
&\dot{\mathbf{q}}_j(t) - f(\mathbf{q}_j(t),\theta(t)) = 0, \\
&\dot{\mathbf{p}}_j(t) + \partial_\q f(\mathbf{q}_j(t),\theta(t))^\top(\mathbf{p}_j(t)) = 0, \\
&\partial_\theta R(\theta(t)) - \sum_{j = 1}^K \partial_\theta f(\mathbf{q}_j(t),\theta(t))^\top(\mathbf{p}_j(t)) = 0 \,, 
       \end{aligned}
 \right\}
 & \text{Hamiltonian equations}
 \end{cases}
\end{align}
with initial conditions $\{(\q_j(0),\p_j(0))\}_{j = 1}^K$, where the gradient with respect to this new parameterization is computed via backpropagation\mnrev{, and typically $K\ll n$.} This set of (position, momentum) pairs is termed the \emph{Hamiltonian particle ensemble}.  As the number of particles is reduced, so are the number of free parameters, see Appendix \ref{appendix:freeparams}. Indeed, varying the Hamiltonian particle ensemble allows for controlling the tradeoff between reconstruction and network complexity. \mnrev{Note that the main difference to the shooting formulation of \eqref{EqOptimalityEquations} is that the parameterization, $\theta(t)$, is now retrieved from the shooting equations as specified by the particle collection. The original data samples, $\tilde{\x}_i$, are simply propagated via these parameters.}

\subsection{Choices of regularization, parameterization and conserved quantities}\label{SecChoiceOfReg}
The main computational bottleneck in the forward model of \eqref{EqShootingForward} is the implicit parameterization of $\theta$ by the last equation. Making it explicit is key to render shooting computationally tractable.

\textbf{Linear in parameter\footnote{Obviously, an affine function of the parameters also works similarly.} - quadratic penalty.} In the simplest case, the space of functions $f$ is a linear space parameterized by $\theta(t)$. In this case, a quadratic penalty amounts to a kinetic penalty. Specifically, as a motivating example, consider the 
forward model
\begin{equation}
f(\x(t),\theta(t)) =  A(t) \sigma(\x(t)) + b(t),
\label{EqAffineVectorField}
\end{equation}
where $\sigma$ is a component-wise activation function,
 $A \in L^2([0,1], \mathbb R^{d^2})$, $b \in L^2([0,1],\mathbb{R}^d)$ and 
$\theta(t) = [A(t),b(t)]$. With the quadratic regularizer $R(\theta(t)) = \frac{1}{2} \Tr\left(A(t)^\top M_{A} A(t)\right) + \frac{1}{2} b(t)^\top M_{b} b(t)$, 
where $M_{A}$, $M_{b}$ are positive definite matrices,
the particle shooting equations are
\begin{equation}
\begin{cases}
\dot \q_j(t) \!&= A(t) \sigma(\q_j(t)) + b(t), \\
\dot{\mathbf{p}}_j(t) \! &=  - \ud \sigma(\q_j(t))^\top A(t)^\top \p_j(t),  \\
\end{cases}\quad
\begin{cases}
A(t) \! &= {M_{A}}^{-1}(-\sum_{j=1}^K \p_j(t) \sigma(\q_j(t))^\top)\\
b(t) \! &= {M_{b}}^{-1}(-\sum_{j=1}^K \p_j(t))\,,
\end{cases}
\label{EqShootingLinearParameterQuadraticPenalization}
\end{equation}
with given initial conditions $(\p_j(0),\q_j(0))$. We emphasize that $\theta(t)$ is \textit{explicitly defined} by $\{(\p_j(t),\q_j(t))\}_{j=1}^K$ and the computational cost is reduced to matrix multiplications.
\par
As is well-known \cite{Arnold1978MathematicalMO}, the Hamiltonian flow preserves the Hamiltonian function. In the ``linear in parameter - quadratic penalty'' case, this preserved quantity, denoted 
$$H(\p(t),\q(t)) = R(\theta(t)),$$
corresponds to a (kinetic) energy of the system of particles. As a first consequence, the objective functional can be rewritten as
$$
H(\p(0),\q(0))) + \gamma \sum_{i=1}^\samplesize \ell(\x_i(T),\tilde{\mathbf{y}}_i)\,.
$$
This clearly allows for direct optimization on $(\p(0),\q(0))$, i.e., shooting.
As a second consequence, since the vector field has constant norm 
(its squared norm is the Hamiltonian), it gives a quantitative bound on the regularity of the flow map at time $t=T$ explicit in terms of $H(\p(0),\q(0))$.
In addition (Appendix \ref{appendix:SecExpectationApproximation}), the Rademacher complexity of the generated flows with bounded $H(\p(0),\q(0)))$ can also be controlled.

\textbf{Nonlinear in parameter and non-quadratic penalty.}
A standard \ResNet\ structure uses vector fields of the type (in convolutional form or not) 
\begin{equation}\label{Eq:SingleHiddenLayer}
f(\x(t),\theta(t)) = \theta_1(t) \sigma(\theta_2(t) \x(t) + b_2(t)) + b_1(t)\enspace,
\end{equation}
where $\theta_1(t) \in L(\R^{d'},\R^d)$ and $\theta_2(t) \in L(\R^d,\R^{d'})$. We will refer to \eqref{Eq:SingleHiddenLayer} as the \textbf{single-hidden-layer} vector field.
This model can also be handled in our shooting approach since the shooting equations in \eqref{EqOptimalityEquations} are completely specified by the Hamiltonian 
$$H(\p,\q,\theta) = R(\theta) - \p^\top f(\q,\theta). $$ Automatic differentiation can be used (see Appendix~\ref{appendix:automatic_shooting}) to implement the forward model
\begin{equation}
\dot{\q}(t) = \frac{\partial H}{\partial \p} (\p(t),\q(t),\theta(t)),~
\dot{\p}(t) = -\frac{\partial H}{\partial \q}(\p(t),\q(t),\theta(t)),~
\theta(t) \in \argmin H(\p(t),\q(t),\theta(t)).
\label{EqTrueHamiltonEquations}
\end{equation}
Note that a necessary condition for solving the third equation above is in fact the compatibility equation in \eqref{EqShootingForward}. 
Important bottlenecks appear since the third equation is nonlinear and potentially associated with a non-convex optimization problem. This could be addressed by unrolling the optimization corresponding to the last equation, resulting in increased computational cost. In addition, in this nonlinear case, the Hamiltonian function is no longer (in general) equal to $R(\theta(t))$ even in the quadratic regularization setting. Therefore, results on the smoothness or Rademacher complexity would no longer be guaranteed as for the linear - quadratic penalty case. Last, quadratic regularization has no known theoretical results for the Rademacher complexity of functions generated by \eqref{Eq:SingleHiddenLayer} with bounded norm. Norms for which the Rademacher complexity of this class of functions is known~\cite{e2019barron} to be bounded are called Barron norms, which are non-smooth and non-convex, and which would add to the difficulty.
To circumvent these issues while retaining expressiveness and theoretical guarantees in the linear parameterization setting, we next introduce the \UpDown\ model.

%

\subsection{The \UpDown\ model}
\label{sec:updownNODE}
The key idea is to transform the vector field of \eqref{Eq:SingleHiddenLayer} into a model which is linear in parameters on which the quadratic regularization can be applied. 
 To this end, we introduce the additional state 
 $$\v(t) = \theta_2(t) \x(t) + b_2(t)$$ 
 which we differentiate with respect to time to obtain $$\dot{\v}(t) = \dot{\theta}_2(t) \x(t) + \dot{b}_2(t) + \theta_2(t) \dot{\x}(t)\enspace.$$ 
 Replacing $\dot{\x}(t)$ by its formula, we get
$$
\dot{\v}(t) = \dot{\theta}_2(t) \x(t) + \dot{b}_2(t) + \theta_2(t) (\theta_1(t) \sigma (\v(t)) + b_1(t))\enspace.
$$
 Now overloading on notation slightly, we use the additional state variable $\v(t)$ to propose the following ODE system, denoted the \UpDown\ model:
\begin{equation}
\dot{\x}(t) = \theta_1(t) \sigma(\v(t)) + b_1(t),\quad
\dot{\v}(t) = \theta_2(t) \x(t) + b_2(t) + \theta_3(t) \sigma(\v(t))\,,
\label{updownODE}
\end{equation}
with $\x(t) \in \R^d$, $\v(t) \in \R^{\inflation d}$ and introducing the (integer-valued) \emph{inflation factor} $\inflation \geq 1$.  For the data evolution, $\x_i(0)$ are given by the data $\{\tilde{\x}_i\}$. We parameterize the $\v_i(0)$ using an affine map $g_{\Theta}$, i.e., 
$$
\v_i(0) = g_{\Theta}(\x_i(0)) = \Theta_{12}(\x_i(0)) + b_{12}, 
$$
 where $\Theta_{12} \in L(\R^d,\R^{\inflation d})$ and $b_{12} \in L(\R^{\inflation d})$.
In Appendix \ref{appendix:updown_universal}, we prove the following theorem:
\vskip2ex
\begin{theorem}
Given a time-dependent vector field defined on a compact domain $C$ of $\R^d$, which is time continuous and Lipschitz, we denote by $\varphi(T,\x(0))$ its flow at time $T$ from starting value $\x(0)$. Then, there exists a parameterization of the \emph{\UpDown}\ model for which
its solution is $\varepsilon$-close to the flow, 
$\sup_{\x(0) \in C}\| \varphi(T,\x(0)) - \x(T) \| \leq \varepsilon $.
\label{thm:universal}
\end{theorem}
Notably, in the proof, the dimension of the hidden state $\v$ is used twice: \emph{first}, for having a sufficient number of neurons in \eqref{Eq:SingleHiddenLayer} to approximate a stationary vector field (standard universality property of multilayer perceptron) and,  \emph{second}, for approximating time-dependent vector fields. Therefore, at the cost of introducing a possibly large number of dimensions, the \UpDown\ model is universal in the class of time-dependent NODEs. 
As shown in Appendix \ref{appendix:updown_universal}, this universality result transfers to our shooting formulation. Due to its additional dimensions, it is also likely to be universal in the space of functions (i.e., not necessarily injective). We focus on the \UpDown\ model in our experiments. Note also that while we derived our theory for vector-valued evolutions for simplicity, similar linear in parameter evolution equations can for example be derived for convolutional neural networks.

\section{Experiments}
\label{section:experiments}

Our goal is to demonstrate that it is possible to learn DNNs by optimizing only over the initial conditions of \textit{critical} networks. This is made possible via shooting and efficient via our particle parameterization. A key difference to prior work is that our approach allows to capture time-dependent (i.e., layer-dependent in the discrete setting) parameters {\it without} discretizing these parameters at every time-point. Comparisons to other NODE like methods are not straightforward due to hyper-parameters and different implementations. For consistency, we therefore provide four different formulations (based on the \UpDown\ model of \S\ref{sec:updownNODE}).

\begin{itemize}
\item The \textbf{static direct} model forgoes the Hamiltonian particle ensemble, and instead directly optimizes over \emph{time-constant} parameters: $\theta(t)=\theta$ for all $t$. Everything else, including the \UpDown\ model, stays unchanged. This model is most closely related to NODE~\cite{chen2018neural} and augmented NODE~\cite{dupont2019augmented}.
\item We call our proposed shooting model \textbf{dynamic with particles}. It is parameterized via a set of initial conditions of (position, momentum) pairs, which evolve over time and fully specify $\theta(t)$.
\item The \textbf{static with particles} model is similar to the {\it static direct} model. However, instead of directly optimizing over a \emph{time-constant} $\theta$, it uses a set of (position, momentum) pairs (i.e., particles, as in our {\it dynamic with particles} model above) to parameterize $\theta$ indirectly.
\item 
Finally, we consider the \textbf{dynamic direct} model which uses a piece-wise time-constant $\theta(t)$. It essentially chains together multiple {\it static direct} models and is closely related to a discrete \ResNet\ in the sense that multiple blocks (we use five) are used in succession. However, each block involves time-integrating the \UpDown\ model. While the \textit{dynamic with particles} model captures $\theta(t)$ indirectly via particles and shooting, the \textit{dynamic direct} model requires many more parameters as it represents $\theta(t)$ directly. We show results for the \textit{dynamic direct} model for a subset of the experiments.
\end{itemize}

\begin{figure}[t!]
 \includegraphics[width=0.99\textwidth]{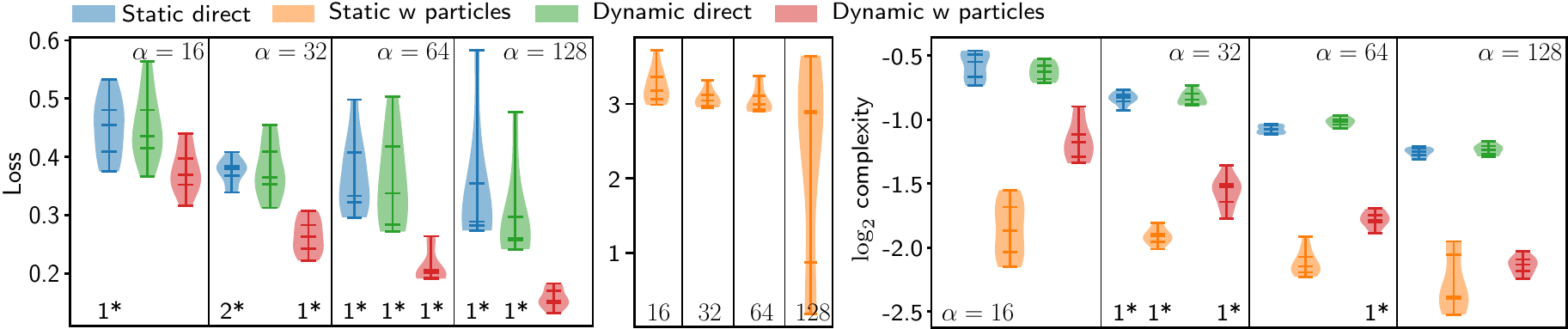}
  \caption{Fit for quadratic-like $y = x^2 + 3/(1+x^2)$ for 10 random initializations. \emph{Left}: Test loss; \emph{Right}: time-integral of $\log_2$ of the Frobenius norm complexity. Lower is better for both measures. * indicates number of removed outliers (outside the interquartile range (IQR) by $\geq1.5\times$ IQR); $\inflation$ denotes the inflation factor. \label{fig:loss_and_complexity_quadratic}}
\end{figure}
All experiments use the \UpDown\ model with quadratic penalty function $R$. Detailed experimental settings, including weights for the quadratic penalty function, can be found in Appendix \ref{appendix:experimental_settings}.

\textbf{Simple 1D function regression.}
We approximate a simple quadratic-like function $y= x^2 + 3/(1+x^2)$ which is non-invertible. We use 15 particles for our experiments. Fig.~\ref{fig:loss_and_complexity_quadratic} shows the test loss and the network complexity, as measured by the log Frobenius norm integrated over time~\cite{neyshabur_exploring_2017}, for the different models as a function of the inflation factor $\inflation$ (cf. \S\ref{sec:updownNODE}). On average, the \textit{dynamic with particles} model shows the best fits with the lowest complexity measures, indicating the simplest network parameterization. Note that the \textit{static with particles} approach results in the lowest complexity measures only because it cannot properly fit the function as indicated by the high test loss. Additional results for a cubic function $y=x^3$ are in Appendix \ref{appendix:figures_results}.

\textbf{Spiral.}
Next, we revisit the spiral ODE example of~\cite{chen2018neural} following the nonlinear dynamics $\dot \x = A\x^3$, $\x \in \mathbb R^2$ (where the power is component-wise). We fix $\x(0)=[2,0]^T$, use $A=[-0.1,2.0;-2,-0.1]$ and evolve the dynamics for time $T=10$. The training data consists of snippets from this trajectory, all of the same length. We use an $L^2$ norm loss (calculated on all intermediate time-points) and 25 particles. Our goal is to show that we can obtain the best fit to the training data due to our dynamic model. Fig.~\ref{fig:spiral_short_long_range} (\emph{top}) shows that we can indeed obtain similar or better fits (lower losses) for a similar number of parameters while achieving the lowest network complexity measures. Fig.~\ref{fig:spiral_short_long_range} (\emph{bottom}) shows the corresponding results for the validation data consisting of the original long trajectory starting from initial value $\x(0)$. Interestingly, by pasting together short-range solutions we are successful in predicting the long-range trajectory despite training on short-range trajectory snippets.
\begin{figure}[t!]
 \includegraphics[width=0.99\textwidth]{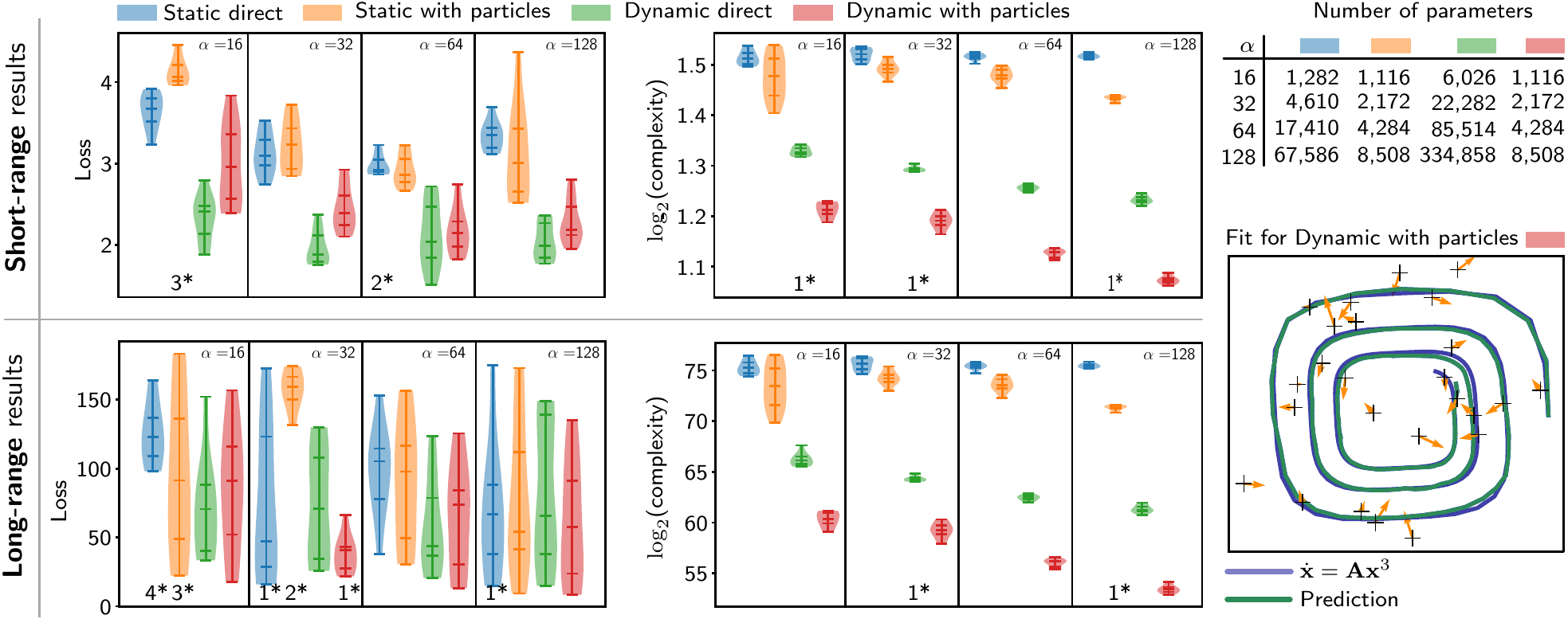}
    \caption{Fit for spiral (short- and long-range). Losses for the different models as well as the time-integral of $\log_2$ of the Frobenius norm complexity measure. Lower is better for both measures. The * symbol indicates how many outliers were removed and $\inflation$ denotes the inflation factor.
  \label{fig:spiral_short_long_range}}
\end{figure}

\textbf{Concentric circles.}
To study the impact of the inflation factor $\inflation$ in a classification regime, we replicate the concentric circles setting of \cite{dupont2019augmented}. 
The task is learning to separate points, sampled from two disjoint annuli in $\mathbb{R}^2$. 
While we are less interested in the learned flow (as in  \cite{dupont2019augmented}), we study how often the proposed \UpDown\ (dynamic with particles) model perfectly fits the training data as a function of $\inflation$.
To the right, we show the success rate over 50 training runs for three choices
\begin{wrapfigure}{r}{4.4cm}
\vspace{-6pt}
\includegraphics[width=4.5cm]{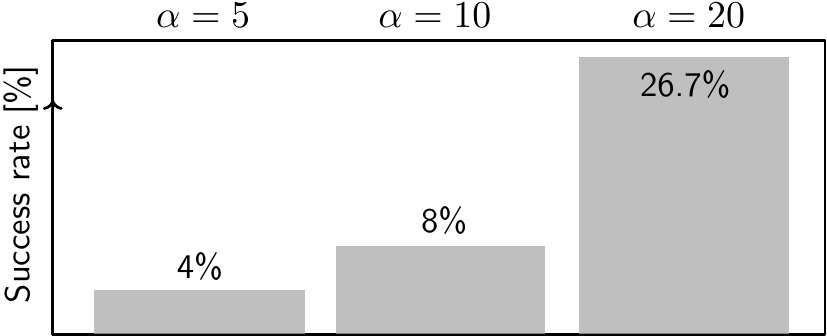}
\vspace{-10pt}
\end{wrapfigure}
of $\inflation$ and 20 particles. \emph{Notably, the effect of $\inflation$ is only visible if the classification loss is down-weighted so that the regularization, $R$, dominates}. 
Otherwise, for the tested $\inflation$, the model always fits the data. The experiment is consistent with~\cite{dupont2019augmented}, where it is shown that increasing the space on which an ODE is solved allows for easy separation of the data and leads to less complex flows. 
The latter is also observed for our model.

\textbf{Rotating MNIST.} Here, we are given sequences of a rotating MNIST digit (along 16 angles, linearly spaced in $[0,2\pi]$). The task is learning to synthesize the digit at any rotation angle, given only the \emph{first} image of a sequence. We replicate the setup of \cite{Yildiz19a} and consider rotated versions of the digit ``3''. We identify each rotation angle as a time point $t_i$ and randomly drop four time points of each sequence during training. One fixed time point is consistently left-out and later evaluated during testing. We use the same convolutional autoencoder  of \cite{Yildiz19a} with the \UpDown\ model operating in the internal representation space after the encoder. 
\begin{wrapfigure}{r}{3.7cm}
	\vspace{-8pt}
	\includegraphics[width=3.7cm]{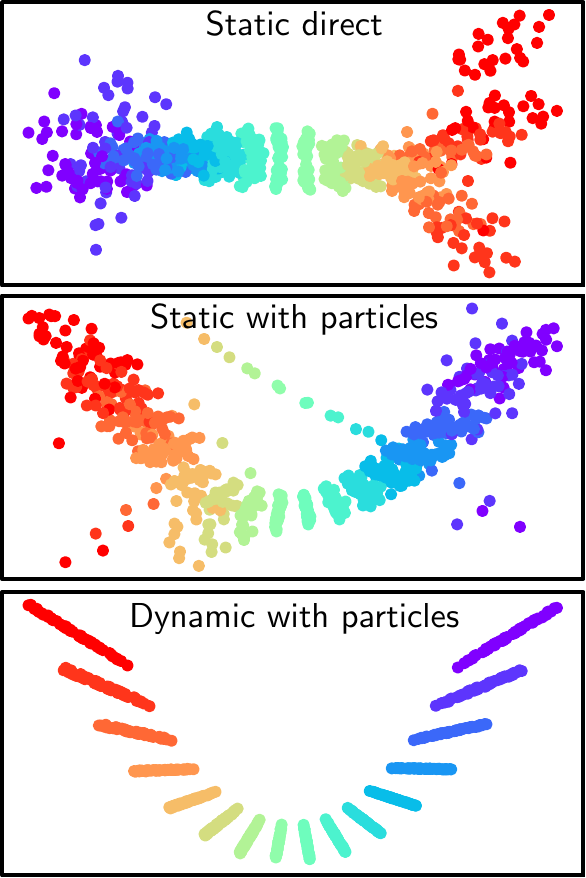}
	\vspace{-8pt}
\end{wrapfigure}
During training, the encoder receives the first
image of a sequence (always at angle $0^\circ$), the \UpDown\ model integrates forward to the desired time points, and the decoder decodes these representations. As loss, we measure the mean-squared-error (MSE) of the decoder outputs. 
Fig.~\ref{fig:mnistresults} lists the MSE (at the left-out angle), averaged over all testing sequences and shows two example sequences with predictions for all time points (100 particles, $\inflation=10$).

While all \UpDown\ variants substantially lower the MSE previously reported in the literature, they exhibit comparable performance. To better understand the differences, we visualize the internal representation space of the autoencoder by projecting all 16 internal representations (i.e., the output of the \UpDown\ models after receiving the output of the encoder) of each testing image onto the two largest principal components, shown to the right (different colors indicate the different rotation angles). This qualitative result shows that allowing for a time-dependent parameterization leads to a more structured latent space of the autoencoder.

\begin{figure*}
\begin{center}
\includegraphics[width=0.99\textwidth]{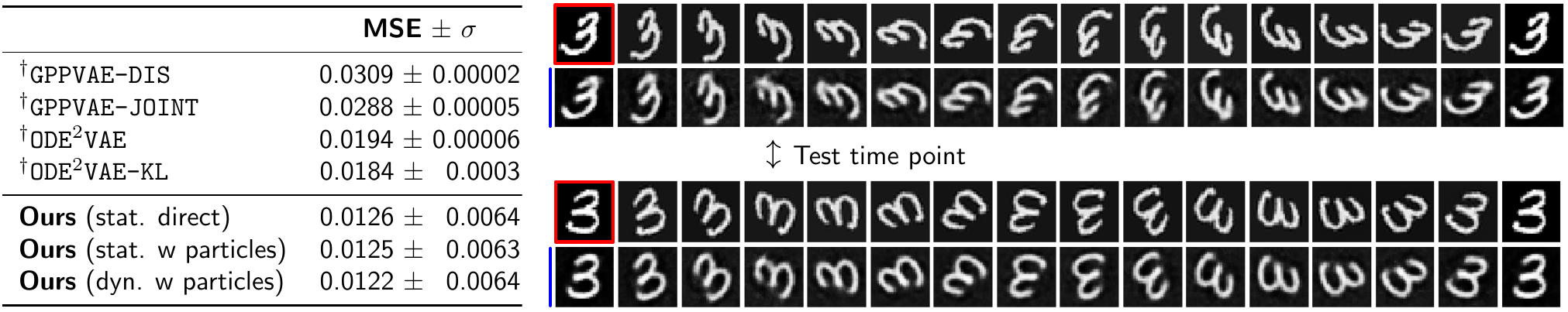}
\end{center}
\caption{\label{fig:mnistresults}\emph{Left}: Image (per-pixel) MSE (measured at the marked time point) averaged over all testing sequences of the rotated MNIST dataset. \emph{Right}: Two testing sequences and predictions (marked \textcolor{blue}{blue}) for all 16 time points when the image at $t=0$ is given as input (marked \textcolor{red}{red}). Results marked with $^\dagger$ are taken from \cite{Yildiz19a}.}
\vspace{-0.25cm}
\end{figure*}

\textbf{Bouncing balls.} 
Finally, we replicate the ``bouncing balls'' experiment
of \cite{Yildiz19a}. This is similar to the rotating MNIST experiment, but the underlying dynamics are more complex. In particular, we are given 10,000 (training) image sequences of bouncing balls at 20 different time points \cite{Sutskever09a}. The task is learning to predict, after seeing the first three images of a sequence, future time points. We use the same convolutional autoencoder of \cite{Yildiz19a} and minimize image (per-pixel) MSE (using all 20 time points for training). Our 
\UpDown~model operates in the internal representation space of the encoder (50-dimensional in our experiments\footnote{We did not further experiment with this hyperparameter, so potentially better results can be obtained.}). In test mode, the network receives the first three image of a sequence and predicts 10 time points ahead. We measure the image (per-pixel) MSE and average the results (per time point) over all 500 testing sequences. For model selection, we rely on the provided validation set. Our \UpDown~(dynamic with particles) model uses 100 particles. Fig.~\ref{fig:bballs_results} (\emph{left}) lists the averaged MSE per time point, plotted against the approaches listed in \cite{Yildiz19a}. Fig.~\ref{fig:bballs_results} (\emph{right}) shows two testing sequences with predictions (the three input time points are not shown). Results for the \UpDown\ static and static with particles model are $\oslash~0.0154$ and $\oslash~0.0150$, respectively.
\begin{figure*}
\begin{center}
\includegraphics[width=0.99\textwidth]{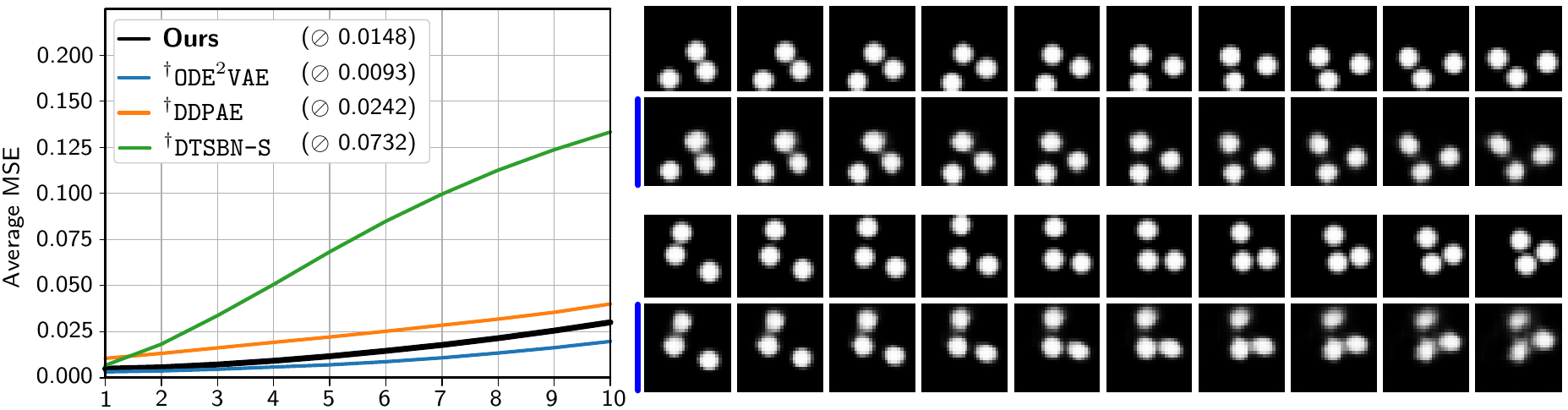}
\end{center}
\caption{\label{fig:bballs_results} \emph{Left}: Image (per-pixel) MSE for predicting 10 time points ahead (after receiving the first three inputs of a sequence), averaged over all testing sequences (numbers in parentheses indicate the MSE when additionally averaged over \emph{all} prediction time points). Results marked with $^\dagger$ are taken from \cite{Yildiz19a}. \emph{Right}: Two testing sequences with predictions (marked \textcolor{blue}{blue}).}
\end{figure*}

\textbf{Computational cost.}
The computational cost of the \UpDown\ model consists in storing the particles and running forward the model for the collection of particles and the data. Hence, computational cost scales linearly in the number of particles. To get rid of this linear relationship (in case only a forward pass is needed), the ODE can be discretized in time and the \texttt{ResNet} with its weights is obtained.

\vspace{-0.2cm}
\section{Discussion and Conclusions}

We demonstrated that it is possible to parameterize DNNs via initial conditions of (position, momentum) pairs. While our experiments are admittedly still simple, results are encouraging as they show that 1) the particle-based approach can achieve competitive performance over direct parameterizations and that 2) time-dependent parameterizations are useful for obtaining simpler networks and can be realized with significantly fewer parameters using particle-based shooting. 

Our work opens up many different follow-up questions and formulations. For example, we presented our approach for a model with continuous dynamics, but the particle and the shooting formalism can also be applied to discrete-time models. Further, we focused, for simplicity, on continuous variants of multi-layer perceptrons, but similar linear-in-parameter models can be formulated for convolutional neural networks. 
Models that are nonlinear in their parameters hold the promise for connections with optimal mass transport theory and to theoretical complexity results, which we touched upon for our \UpDown\ model. Indeed, this change of paradigm in the parameterization may result in new quantitative results on network generalization properties. Lastly, how well the approach generalizes to more complex problems, how many particles are needed to switch from a standard deep network to its shooting formulation, and how optimizing over critical points of the original optimization problem via shooting relates to network generalization will be fascinating to explore.

\textbf{Source code} is available at: \href{https://github.com/uncbiag/neuro_shooting}{\texttt{https://github.com/uncbiag/neuro\_shooting}}

\vspace{-0.2cm}
\section*{Broader Impact}
One goal of this work is to enrich the understanding of continuous depth neural networks and to open a different (or alternative) perspective on its parameterization.
Specifically, we shift the parameterization of deep neural networks from a layer-by-layer perspective to an initial-value perspective and Hamiltonian dynamics. At this point, our work is conceptual and theoretical in nature; broader impact emerges most likely as a consequence of better understanding the role of neural network parameterizations.




\vspace{-0.2cm}
\begin{ack}
This research project was initiated during a one-month invitation of M. Niethammer by the Labex Bézout, supported by  the French National Research Agency ANR-10-LABX-58.
R. Kwitt is partially funded by the Austrian Science Fund (FWF): project FWF P31799-N38 and the Land Salzburg (WISS 2025) under project numbers 20102-F1901166-KZP and 20204-WISS/225/197-2019. S. Wei is the recipient of an Australian Research Council Discovery Early Career Award (project number DE200101253) funded by the Australian Government.


\end{ack}

\bibliography{biblio}

\begin{thebibliography}{10}

\bibitem{Cybenko1989}
Approximation by superpositions of a sigmoidal function.
\newblock {\em Math. Control Signals Syst.}, 2(4):303--314, 1989.

\bibitem{Vialard2012}
Diffeomorphic {3D} image registration via geodesic shooting using an efficient
  adjoint calculation.
\newblock {\em IJCV}, 97(2):229--241, 2012.

\bibitem{br2017optnet}
B.~Amos and J.~Z. Kolter.
\newblock {OptNet}: Differentiable optimization as a layer in neural networks.
\newblock In {\em ICML}, 2017.

\bibitem{Arnold1978MathematicalMO}
V.~I. Arnold.
\newblock {\em Mathematical Methods of Classical Mechanics}.
\newblock Springer, 1978.

\bibitem{athans2013optimal}
M.~Athans and P.~L. Falb.
\newblock {\em Optimal control: an introduction to the theory and its
  applications}.
\newblock Courier Corporation, 2013.

\bibitem{MAL-015}
F.~Bach, R.~Jenatton, J.~Mairal, and G.~Obozinski.
\newblock Optimization with sparsity-inducing penalties.
\newblock {\em Foundations and Trends in Machine Learning}, 4(1):1--106, 2012.

\bibitem{bai2019deep}
S.~Bai, J.~Z. Kolter, and V.~Koltun.
\newblock Deep equilibrium models.
\newblock In {\em NeurIPS}, 2019.

\bibitem{bartlett}
P.~L. Bartlett and S.~Mendelson.
\newblock Rademacher and gaussian complexities: Risk bounds and structural
  results.
\newblock {\em J. Mach. Learn. Res.}, 3(null):463–482, Mar. 2003.

\bibitem{behrmann2018invertible}
J.~Behrmann, W.~Grathwohl, R.~T. Chen, D.~Duvenaud, and J.-H. Jacobsen.
\newblock Invertible residual networks.
\newblock In {\em ICML}, 2019.

\bibitem{benamou2000computational}
J.-D. Benamou and Y.~Brenier.
\newblock A computational fluid mechanics solution to the {M}onge-{K}antorovich
  mass transfer problem.
\newblock {\em Numerische Mathematik}, 84(3):375--393, 2000.

\bibitem{bonnans2013shooting}
J.~F. Bonnans.
\newblock The shooting approach to optimal control problems.
\newblock In {\em IFAC International Workshop on Adaptation and Learning in
  Control and Signal Processing}, 2013.

\bibitem{CompletenessDiffeomorphismGroup}
M.~Bruveris and F.-X. Vialard.
\newblock On completeness of groups of diffeomorphisms.
\newblock {\em J. Eur. Math. Soc. (JEMS)}, 19(5):1507--1544, 2017.

\bibitem{brysonapplied}
A.~Bryson and Y.~Ho.
\newblock Applied optimal control: optimization, estimation, and control. 1975.
\newblock {\em Hemisphere, New York}, pages 177--203.

\bibitem{chen2018neural}
R.~T.~Q. Chen, Y.~Rubanova, J.~Bettencourt, and D.~K. Duvenaud.
\newblock Neural ordinary differential equations.
\newblock In {\em NeurIPS}, 2018.

\bibitem{dupont2019augmented}
E.~Dupont, A.~Doucet, and Y.~W. Teh.
\newblock Augmented neural {ODEs}.
\newblock In {\em NeurIPS}, 2019.

\bibitem{e2019barron}
W.~E, C.~Ma, and L.~Wu.
\newblock Barron spaces and the compositional function spaces for neural
  network models.
\newblock {\em arXiv}, 2019.
\newblock \url{https://arxiv.org/abs/1906.08039}.

\bibitem{finlay_how_2020}
C.~Finlay, J.-H. Jacobsen, L.~Nurbekyan, and A.~M. Oberman.
\newblock How to train your neural {ODE}: the world of {Jacobian} and kinetic
  regularization.
\newblock In {\em ICML}, 2020.

\bibitem{gholami2019anode}
A.~Gholami, K.~Keutzer, and G.~Biros.
\newblock {ANODE}: Unconditionally accurate memory-efficient gradients for
  neural odes.
\newblock In {\em IJCAI}, 2019.

\bibitem{griewank2008evaluating}
A.~Griewank and A.~Walther.
\newblock {\em Evaluating derivatives: principles and techniques of algorithmic
  differentiation}, volume 105.
\newblock SIAM, 2008.

\bibitem{haber2017stable}
E.~Haber and L.~Ruthotto.
\newblock Stable architectures for deep neural networks.
\newblock {\em Inverse Probl.}, 34(1):014004, 2017.

\bibitem{he2016deep}
K.~He, X.~Zhang, S.~Ren, and J.~Sun.
\newblock Deep residual learning for image recognition.
\newblock In {\em CVPR}, 2016.

\bibitem{he2016identity}
K.~He, X.~Zhang, S.~Ren, and J.~Sun.
\newblock Identity mappings in deep residual networks.
\newblock In {\em ECCV}, 2016.

\bibitem{jacobsen2018revnet}
J.-H. Jacobsen, A.~Smeulders, and E.~Oyallon.
\newblock {i-RevNet}: Deep invertible networks.
\newblock In {\em ICLR}, 2018.

\bibitem{li2017maximum}
Q.~Li, L.~Chen, C.~Tai, and E.~Weinan.
\newblock Maximum principle based algorithms for deep learning.
\newblock {\em JMLR}, 18(1):5998--6026, 2017.

\bibitem{liu2019deep}
G.-H. Liu and E.~A. Theodorou.
\newblock Deep learning theory review: An optimal control and dynamical systems
  perspective.
\newblock {\em arXiv}, 2019.
\newblock \url{https://arxiv.org/abs/1908.10920}.

\bibitem{massaroli_dissecting_2020}
S.~Massaroli, M.~Poli, J.~Park, A.~Yamashita, and H.~Asama.
\newblock Dissecting neural {ODEs}.
\newblock {\em arXiv}, 2020.
\newblock \url{https://arxiv.org/abs/2002.08071}.

\bibitem{RademacherVectorValued}
A.~Maurer.
\newblock A vector-contraction inequality for rademacher complexities.
\newblock In R.~Ortner, H.~U. Simon, and S.~Zilles, editors, {\em Algorithmic
  Learning Theory}, pages 3--17, Cham, 2016. Springer International Publishing.

\bibitem{Miller2006}
M.~Miller, A.~Trouv{\'e}, and L.~Younes.
\newblock Geodesic shooting for computational anatomy.
\newblock {\em J. Math. Imaging Vis.}, 24:209--228, 2006.

\bibitem{neyshabur_exploring_2017}
B.~Neyshabur, S.~Bhojanapalli, D.~Mcallester, and N.~Srebro.
\newblock Exploring generalization in deep learning.
\newblock In {\em NeurIPS}. 2017.

\bibitem{Niethammer2011}
M.~Niethammer, Y.~Huang, and F.-X. Vialard.
\newblock Geodesic regression for image time-series.
\newblock In {\em MICCAI}, 2011.

\bibitem{press2007numerical}
W.~H. Press, S.~A. Teukolsky, W.~T. Vetterling, and B.~P. Flannery.
\newblock {\em Numerical recipes 3rd edition: The art of scientific computing}.
\newblock Cambridge University Press, 2007.

\bibitem{reshniak2019robust}
V.~Reshniak and C.~Webster.
\newblock Robust learning with implicit residual networks.
\newblock {\em arXiv}, 2019.
\newblock \url{https://arxiv.org/abs/1905.10479}.

\bibitem{ruthotto2019deep}
L.~Ruthotto and E.~Haber.
\newblock Deep neural networks motivated by partial differential equations.
\newblock {\em J. Math. Imaging Vis.}, pages 1--13, 2019.

\bibitem{saemundsson2019variational}
S.~S{\ae}mundsson, A.~Terenin, K.~Hofmann, and M.~P. Deisenroth.
\newblock Variational integrator networks for physically meaningful embeddings.
\newblock {\em arXiv}, 2019.
\newblock \url{https://arxiv.org/abs/1910.09349}.

\bibitem{Sutskever09a}
I.~Sutskever, G.~Hinton, and G.~Taylor.
\newblock The recurrent temporal restricted {Boltzmann} machine.
\newblock In {\em NeurIPS}, 2009.

\bibitem{toth2019hamiltonian}
P.~Toth, D.~J. Rezende, A.~Jaegle, S.~Racani{\`e}re, A.~Botev, and I.~Higgins.
\newblock Hamiltonian generative networks.
\newblock In {\em ICLR}, 2020.

\bibitem{trefethen1996finite}
L.~N. Trefethen.
\newblock Finite difference and spectral methods for ordinary and partial
  differential equations.
\newblock 1996.

\bibitem{troutman2012variational}
J.~L. Troutman.
\newblock {\em Variational calculus and optimal control: optimization with
  elementary convexity}.
\newblock Springer Science \& Business Media, 2012.

\bibitem{wainwright2019high}
M.~J. Wainwright.
\newblock {\em High-Dimensional Statistics: A Non-asymptotic Viewpoint},
  volume~48.
\newblock Cambridge University Press, 2019.

\bibitem{Yildiz19a}
C.~Yildiz, M.~Heinonen, and H.~Lahdesmaki.
\newblock {ODE}$^2${VAE}: Deep generative second order {ODEs} with {Bayesian}
  neural networks.
\newblock In {\em NeurIPS}, 2019.

\bibitem{laurentbook}
L.~Younes.
\newblock {\em Shapes and Diffeomorphisms}.
\newblock Springer, 2010.

\bibitem{younes2018diffeomorphic}
L.~Younes.
\newblock Diffeomorphic learning.
\newblock {\em arXiv}, 2018.
\newblock \url{https://arxiv.org/abs/1806.01240}.

\bibitem{zhang_approximation_2020}
H.~Zhang, X.~Gao, J.~Unterman, and T.~Arodz.
\newblock Approximation {Capabilities} of {Neural} {ODEs} and {Invertible}
  {Residual} {Networks}.
\newblock In {\em ICML}, 2020.

\bibitem{zhang2019anodev2}
T.~Zhang, Z.~Yao, A.~Gholami, J.~E. Gonzalez, K.~Keutzer, M.~W. Mahoney, and
  G.~Biros.
\newblock {ANODEV2}: A coupled neural {ODE} framework.
\newblock In {\em NeurIPS}, 2019.

\end{thebibliography}
\bibliographystyle{abbrv}

\newpage

\appendix

\section*{Supplementary material}

The following sections discuss in more detail the theoretical guarantees of our approach.  \S\ref{appendix:SecExpectationApproximation} presents the optimality conditions underlying our shooting formulation and it is shown how these optimality equations  can be approximated via a collection of particles. \S\ref{appendix:Barron} proposes different regularizations, whose choice is key for practical and theoretical results. We show that, under some conditions, the Rademacher complexity of the set of flows can be bounded and apply our results in \S\ref{SecAppendixConsequencesUpDown} to the {\UpDown} model. \S\ref{appendix:freeparams} discusses the number of free parameters of our shooting approach in relation to the number of free parameters for direct optimization. \S\ref{appendix:automatic_shooting} explains how the shooting equations can be automatically derived via automatic differentiation.  \S\ref{appendix:updown_universal} shows the universality of our {\UpDown} model. \S\ref{appendix:experimental_settings} provides details on our experimental setup. Lastly, \S\ref{appendix:figures_results} shows some additional experimental results.

\section{Expectation approximation of optimality equations}
\label{appendix:SecExpectationApproximation}

We first discuss a general variational setup of supervised learning including regularization.

\subsection{Variational setup}
Suppose the data consists of input $X \in \R^d$. 
Let $f(\cdot,\theta(t))$ be a vector field on $\mathbb \R^d$, e.g. the single hidden layer of \eqref{Eq:SingleHiddenLayer} or a linear (in parameter) layer. Consider the flow $\varphi := \varphi(T,\cdot)$ generated by $f$ according to
\begin{equation}\label{EqDiffeomorphicFlows}
\begin{cases}
\frac{\ud}{\ud t} \varphi(t,\x) = f(\varphi(t,\x),\theta(t))\,,\\
\varphi(0,\x) = \x\,.
\end{cases}
\end{equation}

We consider the general task of minimizing,
\begin{equation}
\on{Reg}(\varphi) + \gamma \mathbb{E}[\ell(\varphi(X))]\,,
\label{eq:variational_task}
\end{equation}
where $\gamma$ is a positive regularization parameter. 
\par
We now consider the particular case of a \ResNet\ model where each layer is given by an \UpDown\ model \eqref{updownODE} or even a single hidden layer \eqref{Eq:SingleHiddenLayer}.

\par
Without loss of generality, set the terminal time to $T=1$.
Letting $\rho_0$ denote the probability density of $X$, minimizing \eqref{eq:variational_task} is equivalent to minimizing
\[
\inf_{\varphi}\left[ \on{Reg}(\varphi) + \gamma \int_{\R^d} \ell(\varphi(\x)) \rho_0(\x) \ud \x\ \right]\enspace.
\]
This can be rewritten as 
\[
\inf_{\varphi} \left[\on{Reg}(\varphi) + \gamma  \int_{\R^d} \ell(\x') \rho_1(\x') \ud \x'\right]\enspace,
\]
where $\rho_1(\x):=\rho(1,\x)$ is the flow of the continuity equation
\[
\partial_t \rho(t,\x) + \on{div}(\rho(t,\x) f(\x,\theta)) = 0\,,\rho(0,\x)=\rho_0(\x)\,,
\]
where $\on{div}$ is the divergence operator on vector fields. Note that $\rho_1$ can be regarded as the density representing the data at time $1$.
In the following, we deal with a general regularization term $\on{Reg}(\varphi) = \int_0^1 R(\theta(t),\rho(t)) \ud t$, where the $R$ term can depend on the density of data at time $t$. A particular though important case is when the regularization $R$ does not depend on $\rho$, 
$\int_0^1 R(\theta(t)) \ud t$.

\subsection{Optimality equations and Hamiltonian ensemble approximation}
We detail the optimality equations when data points are represented by a probability measure.
As mentioned above, the regularity of the map is enforced via a penalty on the weights at each timepoint and is  the integral $\int_0^1 R(\theta(t)) \ud t$ or even more generally $\int_0^1 R(\theta(t),\rho(t)) \ud t$.
Using Lagrange multipliers, this constraint can be enforced and minimizers of the energy should be saddlepoints of the energy
\begin{multline*}
\mathcal{L}(\rho,\theta,p) := \gamma \int_{\R^d} \ell(\x) \rho_1(\x) \ud \x+ \int_{0}^1 R(\theta(t),\rho(t))\ud t \\ + \int_0^1 \int_{\R^d}  p(t,\x) (\partial_t \rho(t,\x) + \on{div}(\rho(t,\x) f(\x,\theta(t))))  \ud \x \ud t\,,
\end{multline*}
where $p(t,\x)$ is a time and space dependent function. 
The optimality equations are then
\begin{equation}\label{EqOptimalityEquations_density}
\begin{cases}
\partial_t \rho(t,\x) + \on{div}(\rho(t,\x) f(t,\x,\theta(t))) = 0\,,\\
\partial_t p(t,\x) + \nabla p(t,\x) \cdot f(t,\x,\theta(t)) = \frac{\delta R}{\delta \rho}(\theta(t),\rho(t))\,,\\
 \partial_{\theta}R
 (\theta(t),\rho(t)) - \int_{\R^d} 
 \partial_{\theta}f(\x,\theta(t))^\top(\nabla p(t,\x) \rho(t,\x)) = 0\,,
\end{cases}
\end{equation}
where $\nabla p$ is the gradient w.r.t. $\x$ of $p(t,\x)$ and $\delta$ denotes differentiation w.r.t. the indicated parameter. The notation $\frac{\delta R}{\delta \rho}$ means the Fréchet derivative of the penalty w.r.t. the density $\rho$. Note that in our current work, $R$ is independent of $\rho$. However, this more general setup encompasses optimal transport models, see Section \ref{SecNonlinearInParameters}.

In practice, one does not have access to the full distribution and the variational setup needs to be approximated. As proposed in the main text, we approximate it using a collection of particles that follow the optimality equations which are Hamiltonian evolution equations for this collection of particles. The collection of particles $\{(\q_j,\p_j)\}$ are defined by their state and costate. We estimate $\rho$ using the empirical measure $\frac{1}{K}  \sum_{j = 1}^K \delta_{\q_j(t)}(\cdot)$. Writing the optimality equation for this particular empirical measure leads to the equation \eqref{EqShootingForward}.
When the number of particles tends to infinity, we can hope to recover the optimal trajectory.
However, we do not explore this question formally here. We simply remark that this question is directly connected to expressiveness and generalization properties of the constructed neural network and is also probably data dependent.

\section{Choice of regularization}
\label{appendix:Barron}
The simplest regularization on the flow $\varphi$ is given by
\begin{equation}
	\on{Reg}(\varphi) = \int_0^1 R(\theta(t)) \ud t\,,
\end{equation}
where $R$ does not depend on $\rho(t,\cdot)$.
The first possibility is a quadratic penalty for the single-hidden-layer vector field of \eqref{Eq:SingleHiddenLayer}, where $R(\theta(t)) = \frac 12 \| \theta(t) \|^2_2$ is the Frobenius norm of the parameter $\theta$.
Since the space of vector fields is a \emph{finite} dimensional linear space, it can be endowed with a scalar product, which turns this space into a Reproducing Kernel Hilbert Space (RKHS). Therefore, the linear in parameter - quadratic penalty setting of \S\ref{SecChoiceOfReg} is a particular case of vector fields encoded by $f(\cdot,\theta(t))\in H$, with $H$ a RKHS embedded in $W^{1,\infty}$ vector fields. 
This setting
leverages strong analytical and geometrical foundations \cite{laurentbook,CompletenessDiffeomorphismGroup}:

1) When the activation function is smooth, the resulting vector field is smooth\footnote{I.e smoothness asks for Lipschitz regularity vector field, which ensures existence and uniqueness of the flow.}, and consequently the associated flow map $\varphi$ is guaranteed to be a one-to-one smooth map (i.e., a diffeomorphism). For instance, with the \UpDown\ model, it is a homeomorphism in $(\x,\v)$.
Moreover, the quadratic penalty induces a right-invariant distance on the set of flows generated by \eqref{EqDiffeomorphicFlows} and the distance to identity of the resulting flow can be bounded by  $\on{Reg}(\varphi)$ (see~\cite{laurentbook,CompletenessDiffeomorphismGroup} for more details in a Sobolev setting). 2) When the activation function is of {\ReLU} type, the resulting map is still a $W^{1,\infty}$ one-to-one map (i.e., a homeomorphism) and has Lipschitz regularity.
\par
Another type of regularization for the single-hidden-layer vector field of \eqref{Eq:SingleHiddenLayer} we discuss is based on the Barron norm \cite{e2019barron}:
\begin{equation*}
\| \theta \|_{\mathcal{B}}^2 :=  \frac 1{d'}  \sum_{j = 1}^{d'}\| \theta_1^j \|_2^2 (\| [\theta_2]_j \|_1 + \| b_2^j \|_1)^2\enspace,
\end{equation*}
where $\theta_1^j$ denotes the $j^{\text{th}}$ column of $\theta_1$ and $[\theta_2]_j$ denotes the $j^{\text{th}}$ row of $\theta_2$.
As discussed in the main text, the reason we might consider a Barron norm penalty for the single-hidden-layer vector field in \eqref{Eq:SingleHiddenLayer} rather than the quadratic penalty is because of its theoretical results. Indeed, the Rademacher complexity is bounded
for the combination of a single-hidden-layer vector field with a Barron norm penalty, but not when combined with a quadratic penalty.

\subsection{Linear in parameters - quadratic energy}
\label{appendix:quadratic_energy}
Now let us examine in detail models that are \emph{linear} in parameters and have \emph{quadratic} energy on parameters: this case is the simplest to be studied, and computationally not as demanding as the nonlinear case. As mentioned above, the set of possible vector fields $f(\cdot,\theta(t))$ is a finite dimensional linear space, which is a reproducing kernel Hilbert space when endowed with an $L^2$ norm. Since all Hilbert norms in finite dimensions are equivalent, this choice of regularization is universal in this class of quadratic penalties.
\begin{enumerate}
	\item The vector field is $f(\cdot,\theta(t)) = \theta \cdot \mathbf{\sigma}$, where $\mathbf{\sigma}$ is a vector of maps.
	In this case, the optimality equation reads 
	\begin{equation*}
	\partial_{\theta}f(\x,\theta(t))^\top
	(\nabla p(t,\x) \rho(t,\x)) = \int_{\R^d} \mathbf{\sigma}(x)^\top(\nabla p(t,\x)  \rho(t,\x))\ud \x \,.
	\end{equation*}
	\item If the penalty $R$ only depends on $\theta$ and is quadratic: $R(\theta(t)) = \frac{1}{2}\int_0^1 \| \theta(t) \|^2 \ud t$, then one has $\frac{\delta R}{\delta \theta}(\theta(t),\rho(t)) = \theta(t)$.
\end{enumerate}
Thus, under these two conditions, the parameters are \emph{explicit} in terms of $p$, $\rho$ and $\sigma$:
\begin{equation}\label{EqOptimalityEquationLinearInParameter}
\theta(t) =  \int_{\R^d} \sigma(\x)^\top(\nabla p(t,\x) \rho(t,\x))\ud \x\,.
\end{equation}

Two observations are warranted. First, if, instead of quadratic regularization on the parameters, we were to choose a RKHS norm (in the infinite dimensional case) as penalty, it would result in the introduction of the kernel applied to the R.H.S. of \eqref{EqOptimalityEquationLinearInParameter}.
Second, from \eqref{EqOptimalityEquationLinearInParameter}, one could be tempted to derive an evolution equation for $\theta$. This equation is known as the EPDiff equation \cite{laurentbook} and is unfortunately not a closed equation on the set of parameters $\theta(t)$ themselves. Therefore, our approach is a possible way to approximate it.

An important property of this simple setting is that the norm of the vector field is preserved by the forward model defined by the collection of Hamiltonian particles and it also holds in the continuous setting. As stated in Section \ref{SecChoiceOfReg}, the Hamiltonian is given by $R(\theta(t)) = \frac{1}{2} \Tr\left(A(t)^\top M_{A} A(t)\right) + \frac{1}{2} b(t)^\top M_{b} b(t)$ where $A,b$ are the optimal parameters given by
\begin{equation}
\begin{cases}
A(t) \! &= {M_{A}}^{-1}(-\sum_{j=1}^K \p_j(t) \sigma(\q_j(t))^\top)\\
b(t) \! &= {M_{b}}^{-1}(-\sum_{j=1}^K \p_j(t))\,.
\end{cases}
\end{equation}
The Hamiltonian $R(\theta(t))$ being constant gives a constant norm vector field.

\subsection{Nonlinear in parameters - energy which depends on the distribution}\label{SecNonlinearInParameters}
For exposition purposes, we present two cases of interest which we have not well explored numerically.

\textbf{Example of the Barron norm.}
Obviously, the single-hidden-layer vector field in \eqref{Eq:SingleHiddenLayer} is not linear in parameters. We have already discussed that it is proper in this case to endow the space with norms such as the Barron norm \cite{e2019barron}.
For simplicity, consider the single-hidden-layer vector field in \eqref{Eq:SingleHiddenLayer} without $b_1$, i.e., $f(\x(t),\theta(t)) = \theta_1 \sigma(\theta_2(\x) + b_2)$. A simple upper bound for the Barron norm\footnote{The actual Barron norm is defined as the infimum of the r.h.s. in \eqref{EqBarronNormNotOptimized} on all the possible representations of the function $f(\cdot,\theta)$ as a single-hidden-layer.} is
\begin{equation}\label{EqBarronNormNotOptimized}
\| f(\cdot,\theta)  \|_{\mathcal{B}}^2 :=  \frac 1{d'}  \sum_{j = 1}^{d'}\| \theta_1^j \|_2^2 (\| [\theta_2]_j \|_1 + \|[b_2]_j\|_{1})^2\enspace.
\end{equation}

Again, $\theta_1^j$ denotes the $j^{\text{th}}$ column of $\theta_1$ and $[\theta_2]_j$ denotes the $j^{\text{th}}$ row of $\theta_2$.

Let us consider the case of $R(\theta(t)) = \frac 1 2 \| f(\cdot,\theta)  \|_{\mathcal{B}}^2$. 
In this case, one has the following optimality equations to solve
\begin{align*}
& \theta_1^j (\| [\theta_2]_j \|_1 + \| |[b_2]_j \|_1)^2 =  \int_{\R^d} \sigma([\theta_2]_j\x + [b_2]_j)^\top(\nabla p(t,\x) \rho(t,\x))\ud \x\,,\\
& \| \theta_1^j \|^2_2  (\| [\theta_2]_j \|_1 + \| [b_2]_j\|_1) \partial \| [\theta_2]_j^k \|_1 = \int_{\R^d} 
[\ud \sigma([\theta_2]_j\x + [b_2]_j)(\x_k)]^\top(\nabla p(t,x) \rho(t,\x))\ud \x \,,\\
& \| \theta_1^j \|^2_2  (\| [\theta_2]_j \|_1 + \| [b_2]_j \|_1) \partial \|  [b_2]_k^j \|_1 = \int_{\R^d} 
[\ud \sigma([\theta_2]_j\x + [b_2]_j)(\x_k)]^\top(\nabla p(t,\x) \rho(t,\x))\ud \x \,.
\end{align*}

These equations involve the subdifferential of the $L^1$ norm, and optimization of this type of functions, which involves sparsity, is a well-explored field~\cite{MAL-015}.
We leave experiments with this norm for future work. Note that in this case the norm of the vector field is not equal to the Hamiltonian and it is not a constant of the flow.
\subsection{$L^2$ regularization, optimal transport}
Last, we briefly mention a model that is part of our framework which has the advantage of not specifying the penalty on the space of parameters encoding the vector field. In case there is no obvious norm to be used on the space of vector fields, it is possible to use an $L^2$ type of penalty on the vector fields themselves \mnrev{instead of on the parameters.}

Indeed, one way to be rather independent of the choice of the parameterization of the map consists in introducing a cost that represents the $L^2$ norm of the map. However, $L^2$ depends on the choice of a measure and this measure can be  chosen as the density of the data, $\rho(t,\x)$. More precisely, one can use 
\begin{equation}\label{EqOTLike}
R(f,\rho(t))  = \frac 12 \int_{\R^d} \| f(\x,\theta)\|^2 \rho(t,\x) \ud \x\,.
\end{equation} 
In such a case, this formulation resembles finding an optimal transport (OT) map between $\rho_0$ and $\rho_1$. Specifically,
optimal transport is an optimization problem  which can be solved via a fluid dynamic formulation \cite{benamou2000computational} introducing the kinetic penalty above. However, the two models (OT and the one defined by the regularization \eqref{EqOTLike}) differ since the optimization set for optimal transport is the set of $L^2$ vector fields with respect to measure $\rho$ and the above formulation is a parameterized approximation of this set.

This parameterized approximation needs to retain generalization properties of the optimized map. Note however, that in the limit where the number of neurons goes to infinity, optimal transport will be well-approximated since the optimization is performed on a dense subset of all vector fields. Obviously, fixing the choice to a single-hidden-layer design implies a choice for $d'$ in $\theta_1(t) \in L(\R^{d'},\R^d)$ and $\theta_2(t) \in L(\R^d,\R^{d'})$ of ~\eqref{Eq:SingleHiddenLayer}, which thus gives a regularization of the computed approximation of the optimal transport map.


\textbf{Computational burden.}
In either case of the Barron norm or the optimal transport type of penalty, the implicit equation corresponding to the third equation in \eqref{EqOptimalityEquations_density} has to be solved at each layer of the discretization. 
We experimented with a simple strategy of unrolling the related minimization scheme. An efficient approach to solve such implicit equations will be necessary for practical implementations.

\subsection{Rademacher complexity of bounded energy flows. }\label{SecRademacher}
In this section, given a set of vector fields with bounded Rademacher complexity, we show that the resulting flows also have bounded Rademacher complexity.
The flow of a vector field $f(\cdot,\theta(t))$ is a vector valued map denoted by $\varphi$. 
Let us first treat the case of the Rademacher complexity of a component of the flow map $\varphi^k$.
\begin{theorem}\label{ThRademacher}
	Let $\mathcal{F}$ be a space of vector fields defined on a compact space $C \subset \R^d$. Assume that the Rademacher complexity on $n$ points in $C$ of each component of the vector fields $f^k(t,\cdot)$ for $k = 1,\ldots,d$ is controlled by $M(n,t)$ which depends on $n$, then the Rademacher complexity of each component of the flows at time $1$ is bounded by $\int_0^1 M(n,t) \ud t$. 
\end{theorem}
\begin{proof}
Recall that Rademacher complexity, see \cite{wainwright2019high}, of a class of functions $\mathcal{F}$ is defined as, for $\mathbf{Z} = (\z_1,\ldots,\z_n) \in C$,
\begin{equation*}
\on{Rad}_{\mathbf{Z}}(\mathcal{F})\eqdef \mathbb{E}\left[\sup_{g \in \mathcal{F}} \sum_{i = 1}^n \varepsilon_i g(\z_i)\right]\,,
\end{equation*}
where the $\{\varepsilon_i\}_{i = 1}^n$ are i.i.d. Rademacher random variables. Our hypothesis ensures
$\on{Rad}_{\mathbf{Z}}(\mathcal{F}) \leq M(n)$. Apply the definition of the flow to get 
\begin{equation*}
\varphi(1,\x) = \x + \int_0^1 f(\varphi(t,\x),\theta(t)) \ud t\,.
\end{equation*}
Therefore, 
\begin{align*}
\mathbb{E}\left[\sup_{\varphi \in \mathcal{F}} \sum_{i = 1}^n \varepsilon_i \varphi^k(\z_i)\right] &\leq \on{Rad}_{\mathbf{Z}}(\{\on{Id}\}) + \int_0^1 \mathbb{E}\left[\sup_{f(\cdot,\theta(t))}\sum_{i = 1}^n \varepsilon_i f^k(\varphi(t,\z_i),\theta(t))\right] \ud t\,,\\
&\leq  0 + \int_0^1 M(n,t) \ud t \,.
\end{align*}
In the previous formula, we used the fact that the Rademacher complexity of a set comprised of a single map is zero.
\end{proof}

\begin{corollary}
Let $H$ be a RKHS of vector fields whose kernel $\mathsf{k}$ is bounded on the diagonal $\| \mathsf{k}(\x,\x)^k \|_\infty < \infty$ , then, the set of flows denoted by $\mathcal{F}$ at time $1$ of time-dependent vector fields in $\mathcal B(0,R)$, the ball of radius $R$ centered at the origin satisfies $\on{Rad}_{\mathbf{Z}}(\mathcal{F}) \leq \frac{2R \sqrt{\| \mathsf{k}(\x,\x)^k \|_\infty}}{\sqrt{n}}$, where $\on{Rad}_{\mathbf{Z}}(\mathcal{F})$ is the Rademacher complexity for $n$ points.
\end{corollary}

\begin{proof}
The Rademacher complexity of the ball of radius $R$ in the RKHS $H$ \cite[Lemma~22]{bartlett} is upper bounded: $\on{Rad}_{\mathbf{Z}}(\mathcal{B}(0,R)) \leq \frac{2R \sqrt{\| \mathsf{k}(\x,\x)^k \|_\infty}}{\sqrt{n}}$. We then directly apply Theorem \ref{ThRademacher}.
\end{proof}
A similar result also holds for vector fields generated by the single-hidden-layer vector field in \eqref{Eq:SingleHiddenLayer}, see \cite{e2019barron}.
Last, we note that the result and its proof also hold if one uses the following Rademacher complexity for vector valued functions  \cite{RademacherVectorValued},
\begin{equation*}
\on{Rad}_{\mathbf{Z}}(\mathcal{F})\eqdef \mathbb{E}\left[\sup_{g \in \mathcal{F}} \sum_{i = 1}^n\sum_{j =1}^d \varepsilon_j \varepsilon_i g_j(\z_i)\right]\,,
\end{equation*}
for $g = (g_j)_{j = 1,\ldots,d} \in \mathcal{F}$.

\subsection{Consequences for the  \texttt{UpDown} model}\label{SecAppendixConsequencesUpDown}

We put together the previous results on the \texttt{UpDown} model. First, the space of vector fields endowed with the quadratic penalty on the parameters forms a RKHS. 
The variational formulation implies that the norm of the velocity field generated by a given collection of Hamiltonian particles $\{(\q_j(0),\p_j(0))\}$ is preserved. In addition, this norm can be explicitly computed since the parameters at time $0$ can be computed in terms of $\{(\q_j(0),\p_j(0))\}$. Last, the generated space of maps has a Rademacher complexity which is linearly bounded by this norm. In order to be fully explicit on the constant for the Rademacher complexity, we need to compute $\sup_{x \in C} \| \mathsf{k}(x,x) \|$ where $\mathsf{k}$ is the kernel associated with the RKHS. Without making this quantity explicit here, we simply mention that the bound degrades (i.e. increases) with increasing inflation factor $\alpha$, as it can be expected.

\section{Analysis of the number of free parameters}
\label{appendix:freeparams}

It is instructive to understand the number of parameters for a shooting approach in comparison to the typical approach of optimizing a neural network (where the parameter-dependency at optimality is only considered implicitly at convergence of the numerical solution rather than explicitly during the shooting). We focus on the cases of affine and convolutional layers for illustration.

Consider a DNN with a depth of $L$ layers, where each hidden layer has $P$ parameters. The number of free parameters is then $LP$, compared to $2KS$ where $K$ is the number of active particles, each of them of size $2S$\footnote{For example, $S$ for our {\UpDown} model simply corresponds to the dimension of its state space: $S=(\alpha+1)d$, where $\alpha$ is the inflation factor and $d$ the data dimension. Note that in our experiments with the {\UpDown} model we also learned an affine map from the initial conditions $\x(0)$ to the initial conditions $\v(0)$. Such a map has $\alpha d(d+1)$ parameters. These parameters are included in the table of Fig.~\ref{fig:spiral_short_long_range} and in Tables~\ref{tab:number_of_parameters_simple_functions}/\ref{tab:number_of_parameters_spiral} summarizing the number of mode parameters. However, we will not consider parameters in our discussion here, as they would equally apply to both a shooting and a direct optimization approach and could also be avoided by simply initializing $\v(0)$ to zero. A similar initialization to zero approach is, for example, commonly taken in {\ResNets} when increasing the number of feature channels~\cite{he2016deep}.}. Hence, solutions with less than $LP/(2S)$ particles provide benefits in the number of free parameters. Therefore, as the number of particles is reduced, we may parameterize the DNN with a smaller number of parameters. \emph{Most remarkably, the number of free parameters is always $2KS$ regardless of the number of parameters of a particular layer as the layer parameters are obtained via the shooting equations based on the particle states.} This is a consequence of regularizing the parameters in our loss which couples them across time at optimality. We make this clearer in what follows.

\textbf{Affine layers.} Recall that in our simple example of \S\ref{SecChoiceOfReg} the parameters $\theta(t)=[A(t),b(t)]$ of our affine\footnote{In this section, we mean affine with respect to $(\sigma(\q_j(t)))_{j \in 1,\ldots,K}$.} model are given as
\begin{equation}
  A(t) \! = {M_{A}}^{-1}\left(-\sum_{j=1}^K \p_j(t) \sigma(\q_j(t))^\top\right),\quad
  b(t) \! = {M_{b}}^{-1}\left(-\sum_{j=1}^K \p_j(t)\right)\enspace.
  \label{eq:parameters_for_affine_model}
\end{equation}
Here, $A(t)$ and $b(t)$ have $d^2$ and $d$ parameters, respectively; these parameters are indirectly given by the set of particles $\{(\q_j(t),\p_j(t))\}$ at any given time. Hence, for this model $S=d$ and $P=d(d+1)$. If we assume we have $K$ particles and compare to a discrete layer implementation of this model then the particle-based approach will have less free parameters if
\begin{equation*}
  2Kd < L d(d+1)\enspace.
\end{equation*}
Importantly, the state-space dimension, $d$, only enters the number of free parameters linearly for the particle approach ($2Kd$), while there is a quadratic dependence for direct optimization ($Ld(d+1)$). This is a direct consequence of the optimality condition which couples the parameters $\theta(t)$ across time. One can see this phenomenon in action in~\eqref{eq:parameters_for_affine_model}, where the matrix $A$ is expressed as the sum of matrices $\p_j(t) \sigma(\q_j(t))^\top$ with rank $\leq 1$. Concretely, a particle-based shooting approach uses less parameters if the number of particles $K<L(d+1)/2$. Another interesting observation based on this example is that even if we would have only considered a linear model (i.e., without the bias term, $b(t)$) the number of parameters for the particles would have still remained at $2KS$. This is again a consequence of optimality and of our parameter regularization. Note that this also means that even though our {\UpDown} model 
\begin{equation*}
\dot{\x}(t) = \theta_1(t) \sigma(\v(t)) + b_1(t),\quad
\dot{\v}(t) = \theta_2(t) \x(t) + b_2(t) + \theta_3(t) \sigma(\v(t))\,,
\end{equation*}
has significantly more parameters $\theta(t)=[\theta_1(t),b_1(t),\theta_2(t),b_2(t),\theta_3(t)]$ when directly optimized, this has no direct impact on the number of free parameters of its particle-based parameterization. Only the state-space dimension matters. Concretely, if we were to instead consider a model of the form
\begin{equation*}
\dot{\x}(t) = \theta_1(t) \sigma(\v(t)),\quad
\dot{\v}(t) = \theta_2(t) \x(t)\,,
\end{equation*}
the particle-based parameterization would stay \emph{unchanged!} Only the way how one infers $\theta(t)$ from the particles changes.

\textbf{Convolutional layers.} Shooting approaches for convolutional models can also be derived. We did not experiment with such models in this work. However, we show here that the number of free parameters may also be decreased with a particle-based approach. This will be interesting to explore in future work. Specifically, for convolutional layers a particle-based parameterization could be particularly effective as one typically has quadratic complexity in the number of filters between convolutional layers (i.e., if a layer with $N$ feature channels is followed by a layer with $M$ feature channels, this will induce the estimation of $N\times M$ convolutional filters and hence will drastically influence the number of parameters for large $N$ or $M$). In contrast, a particle-based shooting approach does not increase the number of parameters as it ties them together via the optimality conditions expressed by the shooting equations. As a rough estimate for a standard convolutional \ResNet\, for $L = 50$, $P = 100^2 \times 16$, $LP \approx 8.10^6$. Thus, if particles have size $40$, we end up with at most $10^5$ active particles.

\textbf{General remarks.} Nevertheless, all model parameters (e.g., $[A(t),b(t)]$ or all convolutional filters for a convolutional layer) are still instantiated during computation. It is important to note that regardless of the chosen number of particles, a shooting neural network solution is a possible optimal solution (for a given data set) at any given time, not only at convergence. One optimizes over the family of possible neural network models with the goal of finding the element within this family that best matches the observations.

\section{Automatic shooting}
\label{appendix:automatic_shooting}

The general shooting equations were presented in \eqref{EqOptimalityEquations}. We then proceeded to explicitly derive the shooting equations for a continuous DNN with linear-in-parameter layers and \UpDown\ layers in \S\ref{SecChoiceOfReg} and \S\ref{sec:updownNODE}, respectively. While this was instructive, it is somewhat cumbersome, in particular, for more complex models or when moving to convolutional networks. Fortunately, in practice these shooting equations do not need to be derived by hand. Indeed, they are completely specified by the Hamiltonian 
$$
H(\p,\x,\theta) = \p^\top (\dot{\x} - f(t,\x,\theta)) + R(\theta)\enspace,
$$
in the sense that the shooting equations in \eqref{EqOptimalityEquations} are computed via differentiation of $H$. Specifically, the shooting equations in \eqref{EqOptimalityEquations} are equivalently given by
\begin{align*}
	\begin{cases}
		\dot{\x} = \frac{\partial H(\p,\x,\theta)}{\partial p}\,,  \\
		\dot{\p} = -\frac{\partial H(\p,\x,\theta)}{\partial x}\,,\\
		\theta \in \argmin_{\theta} H(\p,\x,\theta)\,.
	\end{cases}
\end{align*}
As discussed above, the last equation can be replaced by solving 
$$
\partial_\theta R(\theta) - \sum_{i = 1}^N\partial_\theta f(t,\x_i,\theta)^T(\p_i) = 0\enspace.
$$
Automatic differentiation can be used to automatically obtain the shooting equations.
As fitting a shooting model requires differentiating the shooting equations, we in effect end up with differentiating twice. This can be done seamlessly using  modern deep learning libraries, such as \texttt{PyTorch}.

\section{Universality of the \UpDown\ model}
\label{appendix:updown_universal}
In this section, we set out to demonstrate that the \UpDown\ model is universal in the sense that its associated flow can come $\varepsilon$-close to the flow of any well behaving time-dependent vector field. 

Recall the single-hidden-layer vector field in \eqref{Eq:SingleHiddenLayer} with time-varying parameters $\theta(t)=(\theta_1(t),\theta_2(t),b_1(t), b_2(t))$.
While shooting with the single hidden layer vector field is theoretically appealing as it is universal \cite{Cybenko1989}, it would result in implicit shooting equations.
We first show that the \UpDown\ model introduced in \ref{sec:updownNODE} can give the same flow as the single hidden layer (Lemma \ref{lem:updown_shl}) and then leverage this relationship to show that the \UpDown\ model inherits the universality of the single hidden layer (Proposition \ref{prop:updown_universal}). 

 \begin{lemma}\label{ThLemma1}
 Consider the single-hidden-layer vector field in \eqref{Eq:SingleHiddenLayer} with $\theta_2(t)$ and $b_2(t)$ being piecewise $C^1$  and $\theta_1(t),b_1(t)$ continuous. Then, there exists a parameterization of the {\UpDown} model that gives the same flow at a fixed time, $T=1$.
 \end{lemma}
 \begin{proof}
 We rewrite the differential equation
 \begin{equation*}
 \dot{\q}(t) = \theta_1(t) \sigma (\theta_2(t) \q + b_2(t)) + b_1(t)\,,  
 \end{equation*}
 by introducing the additional state variable $\v(t)  = \theta_2(t) \q(t)  + b_2(t)$ which we differentiate w.r.t. time.
 We obtain $\dot{\v}(t)  = \dot{\theta}_2(t) \q(t)  + \dot{b}_2(t) + \theta_2(t) \dot{\q}(t)  \,.$ Replacing $\dot{\q}(t)$ by its formula, we get
 \begin{equation*}
 \dot{\v}(t)  = \dot{\theta}_2(t) \q(t)  + \dot{b}_2(t) + \theta_2(t) \theta_1(t)( \sigma (\v(t)) + b_1(t))\,.
 \end{equation*}
 The system can be rewritten as 

 \begin{equation}
 \begin{cases}\label{EqSystemUpdown}
 \dot{\q}(t)  =  \theta_1(t) \sigma (\v(t)) + b_1(t)\,,\\
 \dot{\v}(t)  = \theta_3(t) \q(t)  + \theta_4(t)\sigma (\v(t)) + b_3(t)\,.
 \end{cases}
 \end{equation}
Therefore, with the initial condition $\v(0) = \theta_2(0)\q(0)$ and $\q(0) = \q_0$, the two systems of ordinary differential equations are equivalent.
 \end{proof}

Note that the key point in Lemma \ref{ThLemma1} is the loss of regularity in the evolution of $\theta_2$ since we differentiated once in time. For that reason, we now show that adding more dimensions using the inflation factor $\inflation$ alleviate this issue.
It is likely possible that one could prove a universality result using only $\inflation = 1$ but we shall leave this question for future work\footnote{Note that the case $\inflation =1$ is similar in its formulation to a second-order model on $\q$.}. However, experimentally, the inflation factor has a crucial effect on the performance of the optimization, as discussed in \S\ref{section:experiments}. Lemma \ref{ThLemma1} helps us establish the next result. 

\begin{lemma}\label{lem:updown_shl}
	Consider the single-hidden-layer vector field in \eqref{Eq:SingleHiddenLayer} with $\theta(t)$ being piecewise continuous. Then, there exists a parameterization of the \UpDown\ model that gives the same flow.
\end{lemma}

\begin{proof}
Without loss of generality, we only treat the case of one discontinuity in time of the parameterization; We thus assume that $\theta(t)$ is continuous on $[0,t_1[$ and $[t_1,1]$. We consider $\q,\v_1,\v_2 \in \R^d$ such that $\q,\v_1$ are defined as in Lemma~\ref{ThLemma1}.
We now define, up to time $t_1$, $\v_2(t) = \theta(t_1) \v_1(t) + \theta_2(t_1)$ which implies (differentiating w.r.t. time) that $\v_2$ follows an evolution equation similar to $\v_1$ and thus can be encoded in the general form of~\eqref{EqSystemUpdown}.
Now, $\q(t),\v_2(t)$ are defined on $[t_1,1]$ by the evolution \eqref{EqSystemUpdown} in order to coincide with the flow of single-hidden-layer vector field on $[t_1,1]$, $ \dot{\q}(t)  =  \theta_1(t) \sigma (\v_2(t)) + b_1(t)$ and $\dot{\v}_2(t)  = \theta_3(t) \q(t)  + \theta_4(t)\sigma (\v_2(t)) + b_3(t)$ for well chosen parameters as in  Lemma~\ref{ThLemma1}. Since the value of $\v_1(t)$ is not used in the evolution equation of $\q(t) $, we can simply extend it by $\v_1(t) = \v_1(t_1)$ which is a valid evolution equation for~\eqref{EqSystemUpdown}.
\par
In the general case, we decompose the time interval $[0,1]$ into $k$ intervals $[t_i,t_{i + 1}[$ on which $\theta(t)$ is continuous and the proposed method can be directly extended using an inflation factor $\inflation = k$, introducing $\v_k \in \R^d$.
\end{proof}
Note that the result of this lemma gives an equality between the two flows defined on the \emph{whole} space $\R^d$. 
The next result is an approximation result which holds on a compact domain $C \subset \R^d$. For a function $f: \R^d \to \R$,  we denote $\| f \|_{C,\infty} = \sup_{x \in C} |f(x)|$.

\begin{proposition}\label{prop:updown_universal}
	The \UpDown\ model is universal in the class of time-dependent vector fields. Let $C \subset \R^d$ be a compact domain. For every time-dependent vector field (such that it is time continuous and is Lipschitz in space) $w: [0,1] \times \R^d \mapsto \R^d$ and its associated flow $\varphi(t,\x)$ there exist time dependent parameters of the \UpDown\ model such that 

	\begin{equation*}
	\begin{cases}
	\dot{\q}(t)  = \theta_1(t)  \sigma(\v(t)) + b_1(t) \,,\\
	\dot{\v}(t)  = \theta_2(t) (\q(t) ) + b_2 + \theta_3(t)  \sigma(\v(t) )\,,
	\end{cases}
	\end{equation*}

	is $\varepsilon$-close to the solution $\varphi(1,\x)$, e.g. $\| \varphi(1,\x) - \q(1,\x) \|_{C,\infty} \leq \varepsilon$.
\end{proposition}

\begin{proof}
	The proof is standard and we include it here for self-containedness. It is the consequence of \cite{Cybenko1989} and Lemma \ref{lem:updown_shl}. Let $B(0,r)$ a ball of radius $r$ in $\R^d$ which contains $\varphi(t,\x)$ for all time $t \in [0,1]$. The flow associated with a given time-dependent vector field $v(t,\cdot)$ can be approximated by a vector field which is piecewise constant in time; i.e. let $\varepsilon >0$ be a positive real, (by continuity in time of $v(t,\cdot)$) there exists a decomposition of $[0,1]$ into $k$ intervals $[t_i,t_{i + 1}]$ and Lipschitz vector fields  $v_i(\x) = f(\x,\theta_i)$ where $f$ is the single hidden layer in \eqref{Eq:SingleHiddenLayer} such that
$\| v_i(\x) - v(t,\x) \|_{B(0,r),\infty} \leq \varepsilon$ for $t \in [t_i,t_{i + 1}]$. Denote by $w(t,\cdot)$ the time-dependent vector field defined by $w(t,\cdot) = v_i(\cdot)$  for all $t \in [t_i,t_{i + 1}]$. Thus, denoting the flow of $v(t,\cdot)$ by $\varphi^v$ and the flow of $w(t,\cdot)$ by $\varphi^w$, we get
\begin{align*}
&\|\varphi^v(1,\x) - \varphi^{w}(1,\x) \| \leq \int_0^1 \| v(t,\varphi^v(t,\x)) - v(t,\varphi^{w}(t,\x))\|  + \| v(t,\varphi^{w}(t,\x)) - w(t,\varphi^{w}(t,\x))\| \ud t \,\\
&\|\varphi^v(1,\x) - \varphi^{w}(1,\x) \|_{C,\infty}  \leq \int_0^1 \on{Lip}(v) \|\varphi^v(t,\x) - \varphi^{w}(t,\x) \|_{C,\infty} +  \| v(t,\cdot) - w(\cdot)\|_{B(0,r),\infty} \ud t\\
& \phantom{\|\varphi^v(1,\x) - \varphi^{w}(1,\x) \|_{K,\infty} } \leq  \int_0^1 \on{Lip}(v) \|\varphi^v(t,\x) - \varphi^{w}(t,\x) \|_{C,\infty} \ud t + \varepsilon\,,
\end{align*}
where $\on{Lip}(v)$ denotes a bound on the Lipschitz constant of $v(t,\x)$ w.r.t. $\x \in B(0,r)$ for all $t\in [0,1]$.
Then, the Gr\"onwall lemma \cite{laurentbook} gives 
\begin{equation}
\|\varphi^v(1,\x) - \varphi^{w}(1,\x) \|_{C,\infty} \leq \varepsilon e^{\on{Lip}(v)}\,.
\end{equation}
By Lemma~\ref{lem:updown_shl}, $\varphi^{w}(1,\x)$ can be approximated by the flow of the \UpDown\ and the result is obtained via the triangle inequality.
\end{proof}

In this section, we focused on a universality result in the space of time-dependent vector fields. Interestingly, due to the additional dimensions, it is likely that the model is universal in the space of functions as well. This conjecture is supported by the quadratic 1D function regression example which shows that the \UpDown\ model is able to capture some maps which are not homeomorphic.
We leave this question for future work.

\section{Experimental settings}
\label{appendix:experimental_settings}

This section describes our experimental settings. We use our {\UpDown} model for all experiments and simply use a weighted Frobenius norm penalty for all parameters. Specifically, we weigh this penalty for all parameters with $1$ except for, $\theta_3$ which we penalize by 10. In our experiments, we have observed better convergence properties for higher penalties on $\theta_3$. This might be due to the special role that $\theta_3$ plays in the model as it subsumes a quadratic term in the original derivation of the {\UpDown} model (see \S\ref{sec:updownNODE}). In all experiments, we also optimize over the affine map from $x(0)$ to $v(0)$ for the data evolution.

\textbf{Simple function regression.} We use 500 epochs for all experiments. For all particle-based experiments we freeze the positions of the particles for the first 50 epochs. We use a {\ReLU} activation function and the MSE loss. We weigh the MSE loss by 100 and the parameter norm loss by 1. We use 500 training samples, 1,000 testing samples and 1,000 validation samples and a batch size of 50. Note that for these simple examples there is, in practice, no real difference between the training, testing, and validation data, as the number of samples is large and the domain is $[-1.5,1.5]$. We initialize the particle positions uniformly at random in $[-1.5,1.5]$ and draw the momenta from a Gaussian distribution with zero mean and standard deviation $0.1$. All time-integrations are done via a fourth-order Runge-Kutta integrator with time-step 0.1. For optimization, we use {\Adam} with a learning rate of 0.01 and the \texttt{ReduceLROnPlateau} learning rate scheduler of {\PyTorch} with a learning rate reduction factor of 0.5.

\textbf{Spiral.} The spiral data is generated between time $t=0$ and $t=10$ with 200 uniformly spaced timepoints. Training is only on small time snippets with an approximate length of $0.25$ time-units. Evaluation is on these short time snippets as well as on the entire trajectory by pasting together solutions for these short time snippets, i.e., an individual short solution starts where the previous one ends. Settings for the spiral are the same as for the simple function regression with the following exceptions. We use 1,500 epochs and the step-size for the fourth-order Runge-Kutta integrator is 0.05. The MSE loss is still weighted by 100, but the parameter norm loss only by 0.01. We randomly draw 100 new training samples during each epoch and use 100 evaluation samples and 1,000 short range samples and 1 long-range testing sample. All samples are randomly drawn from the trajectory. However, as the trajectory is traversed at highly nonuniform speed the samples are drawn from a uniform distribution across the trace of the spiral. As for the simple function regression experiment, there is little practical difference between the training, validation, and testing data as the problem is so simple. However, this is not of concern in these experiments as the prime objective is to study the fitting behavior of the different models. We use a batch size of 100.

\textbf{Rotating MNIST.} We use the data provided by the authors of \cite{Yildiz19a} and follow the same autoencoder architecture, except that our encoder maps into a 20-dimensional 
representation space. The 
number of particles is set to 100 and the inflation factor $\alpha$ is set to 10. For optimization, we use 
{\Adam} with a learning rate of 0.001 and the \texttt{CosineAnnealingLR} learning rate scheduler of {\PyTorch}. We train for 500 epochs with a batch size of 25 and the parameter norm loss set to 0.1.

\textbf{Bouncing balls.} As in the rotating MNIST experiment, we rely on the data provided by the authors of \cite{Yildiz19a},
follow their autoencoder architecture and set the dimensionality
of the representation space of the encoder to 50. The first three images of each sequence are provided to the encoder by concatenating the images along the channel dimension. The inflation factor $\alpha$ is set to 20 and we use 100 particles. 
We optimize over 100 epoch using {\Adam} with the \texttt{CosineAnnealingLR} learning rate scheduler of {\PyTorch}, the initial learning rate is set to 0.001 and the parameter norm loss is set to 0.0001.

\section{Additional results}
\label{appendix:figures_results}


\textbf{Simple function regression.}
In Section \ref{section:experiments}, we considered approximating a quadratic-like function. Here we show parallel results for approximating a cubic function $y=x^3$. We will also include some additional figures for the quadratic-like regression function. Note that whereas the cubic function is invertible (but not diffeomorphic), the quadratic-like one considered in Section~\ref{section:experiments} is a simple example of a non-invertible function. Tab.~\ref{tab:number_of_parameters_simple_functions} shows the number of parameters for the four different formulations for both regression functions. Fig.~\ref{fig:loss_and_complexity_cubic} shows for the cubic regression the test loss and the network complexity, as measured by the Frobenius norm \cite{neyshabur_exploring_2017}, for the four formulations. On average the particle-based approaches show the best fits with the lowest complexity measures, indicating the simplest network parameterization. Note however that while the dynamic particle approach greatly outperformed the static particle approach for the quadratic-like function (see Fig.~\ref{fig:loss_and_complexity_quadratic}) this is not the case here. In fact, the static particle approach shows slightly better fits than the dynamic one. This might be because the cubic function is significantly simpler to fit and hence may not benefit as much from the dynamic approach. To illustrate that fitting the quadratic-like function is indeed harder, Figs.~\ref{fig:cubic_fit_across_particles} and~\ref{fig:quadratic_fit_across_particles} show function fits for different numbers of particles for the cubic function and the quadratic-like function, respectively. All these fits are for the particle-based dynamic {\UpDown} model. Clearly, very few particles can achieve reasonable fits for a simple function. As little as two particles already show a good fit for the cubic function, whereas the quadratic-like function requires with more particles. This supports our hypothesis that fitting more complex functions may require more particles.

Since our approach is based on the time-integration of the {\UpDown} model it is interesting to see 1) how the mapping is expressed across time and 2) how the parameters, $\theta(t)$, of the {\UpDown} model change over time.  Fig.~\ref{fig:data_mapping_simple_functions} shows example mappings for the cubic and the quadratic-like function, respectively. The estimated mappings are highly regular. Lastly, Figs.~\ref{fig:weight_evolution_cubic} and~\ref{fig:weight_evolution_quadratic} show the time-evolutions of the model parameters for the cubic and the quadratic-like function for two different inflation factors. While different parameters show different dynamics, clear changes over time can be observed. In particular, $\theta_2(t)$ and $b_2(t)$ show strong changes. These parameters mostly control the behavior of the hidden high-dimensional state, $v$, as $\theta_3(t)$ is penalized significantly more in our model (see Sec.\ref{appendix:experimental_settings}) and consequently shows more moderate changes.

\begin{table}[h!]
\begin{small}
	\caption{Number of parameters for the simple function regression cubic and quadratic experiments.}
	\label{tab:number_of_parameters_simple_functions}
	\begin{center}
	\vskip1ex
	\begin{tabular}{r|c|cccccc}
		& & \multicolumn{6}{c}{\textbf{Inflation factor}} \\\hline
		& \textbf{\#Particles} & 4 & 8 & 16 & 32 & 64 & 128 \\ \hline
		\multirow{4}{*}{static/dynamic w/ particles} & 2 & 28 & 52 &  100 & 196 & 388 &  772 \\
		& 5 & 58 &  106 &  202 & 394 & 778 & 1,546 \\
		& 15 & 158 & 286 & 542 & 1,054 & 2,078 & 4,126 \\
		& 25 & 258 & 466 & 882 & 1,714 & 3,378 & 6,706 \\ \hline
		dynamic direct & n/a & 153 & 461 & 1,557 & 5,669 & 21,573 & 84,101 \\ \hline
		static direct & n/a & 37 & 105 & 337 & 1,185 & 4,417 & 17,025
	\end{tabular}
	\end{center}
\end{small}
\end{table}

\begin{figure}[h!]
	\centering
	\includegraphics[width=0.9\textwidth]{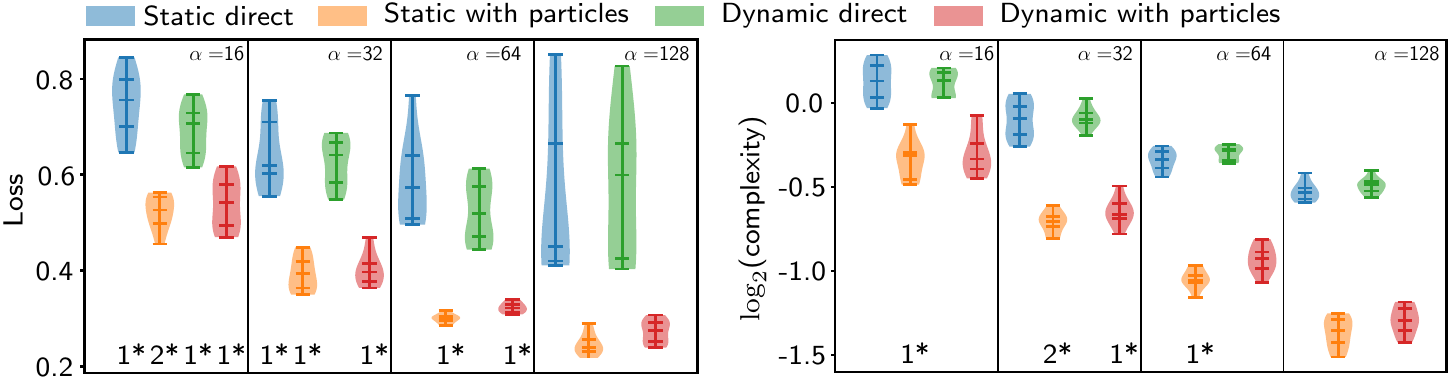}
	\caption{Function fit (15 particles) for cubic $y=x^3$ for 10 random initializations. \emph{Left}: Test loss; \emph{Right}: time-integral of $\log_2$ of the Frobenius norm complexity. Lower is better for both measures. * indicates number of removed outliers (outside the interquartile range (IQR) by $\geq1.5\times$ IQR); $\alpha$ denotes the inflation factor.}
	\label{fig:loss_and_complexity_cubic}
\end{figure}
\begin{figure}[h!]
\begin{center}
\includegraphics[width=0.8\textwidth]{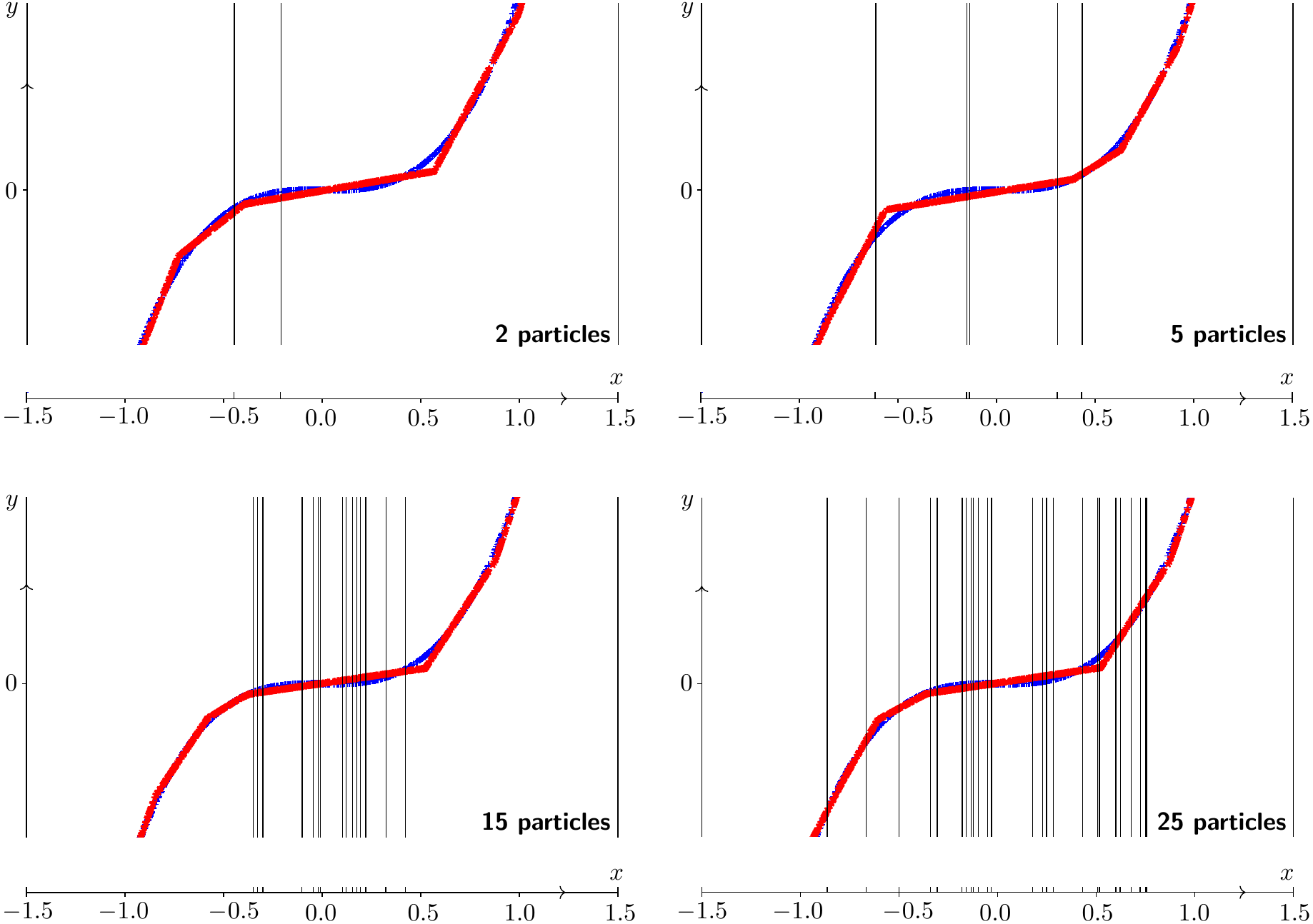}
\end{center}
	\caption{Fits for the \emph{cubic} function with inflation factor 16 and for different numbers of particles. Vertical lines indicate particle positions after optimization. While subtle, the figures suggest that using more particles allows for better approximation of the function. This is confirmed by the test loss values in Fig.~\ref{fig:loss_and_complexity_cubic} (bottom left).}
	\label{fig:cubic_fit_across_particles}
\end{figure}


\begin{figure}[h!]
\begin{center}
\includegraphics[width=0.8\textwidth]{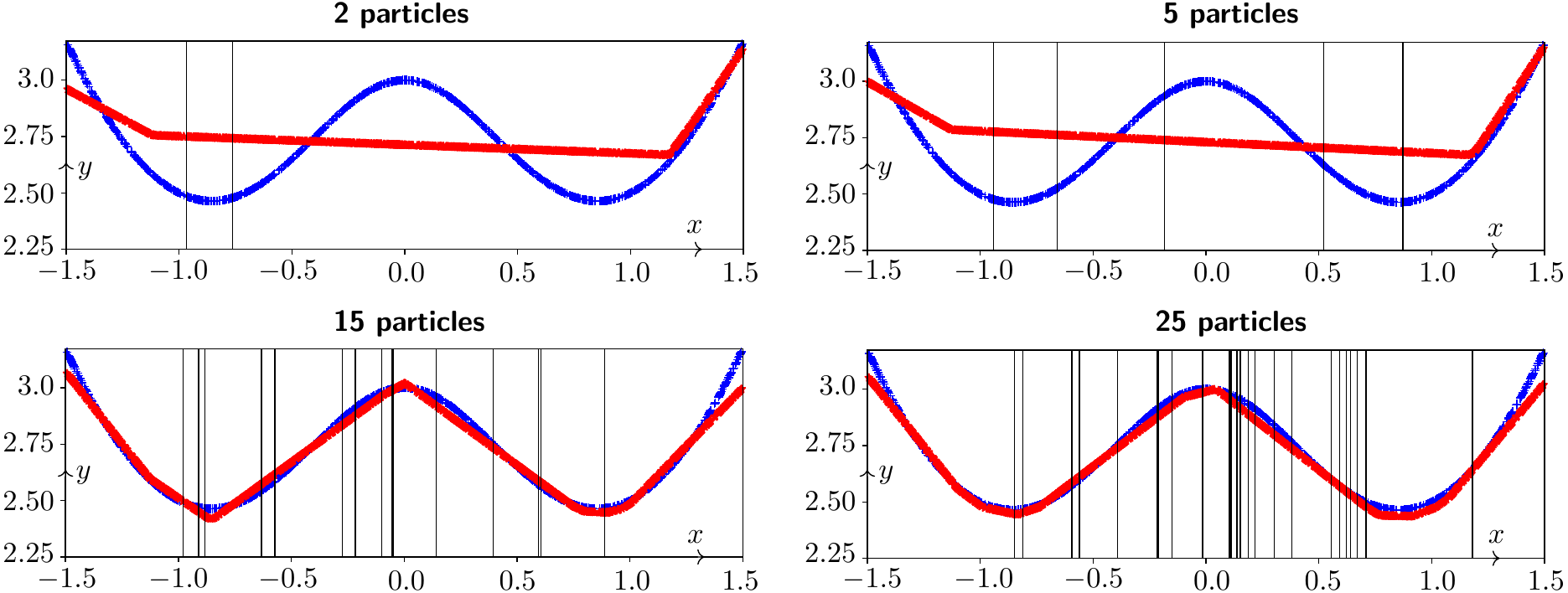}
\end{center}
	\caption{Fits for the \emph{quadratic-like} function for inflation factor 16 with different numbers of particles. Vertical lines indicate particle positions after optimization. As this function is more complex than the cubic function 2 and 5 particles is not sufficient for a fit. But 15 and 25 particles result in a well-fitting approximation.}
	\label{fig:quadratic_fit_across_particles}
\end{figure}


\begin{figure}[h!]
\centering
	\includegraphics[width=0.9\columnwidth]{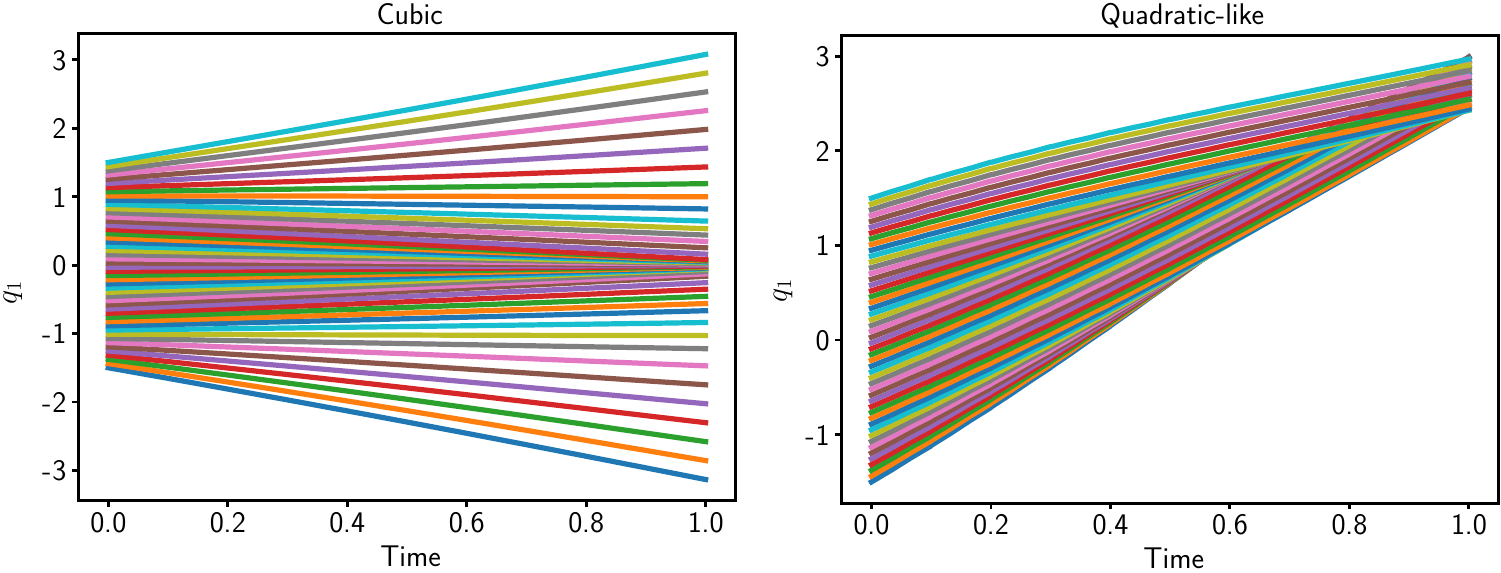}
	\vskip1ex
	\caption{Mapping of the \emph{cubic} function (left) and the \emph{quadratic-like} function (right). As can be seen, the mappings are \emph{highly regular}.}
	\label{fig:data_mapping_simple_functions}
\end{figure}

\begin{figure}
\includegraphics[width=1.0\textwidth]{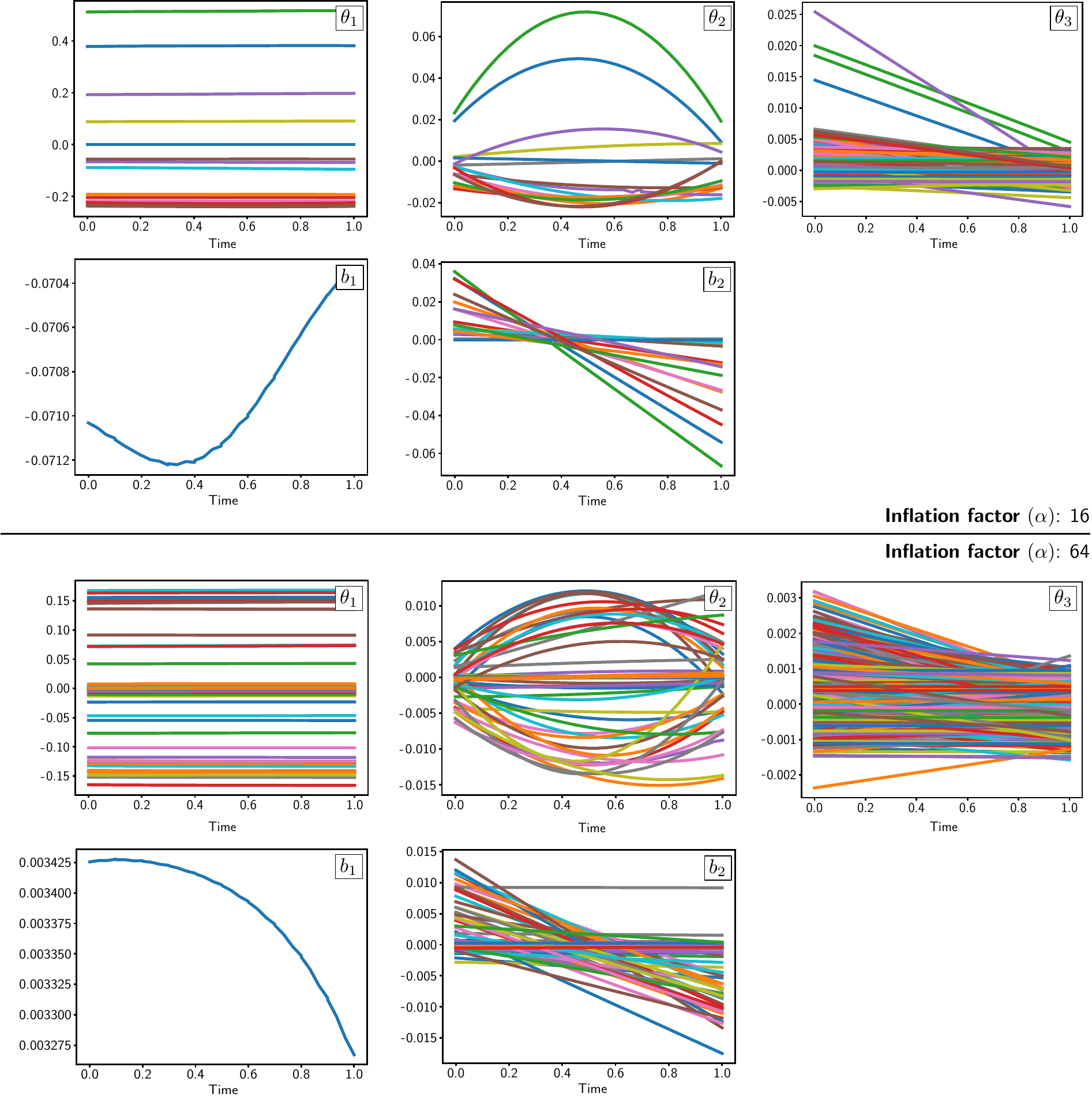}
\caption{Weight evolution across time (i.e., continuous depth) for 15 particles when fitting the \emph{cubic} function using the \UpDown\ model: $\dot{\x}(t) = \theta_1(t) \sigma(\v(t)) + b_1(t),~
\dot{\v}(t) = \theta_2(t) \x(t) + b_2(t) + \theta_3(t) \sigma(\v(t))$. Results are for the dynamic with particles approach. Top: Inflation factor 16. Bottom: Inflation factor 64. Changes in parameter values can clearly be observed. 
}
\label{fig:weight_evolution_cubic}
\end{figure}

\begin{figure}
\includegraphics[width=1.0\textwidth]{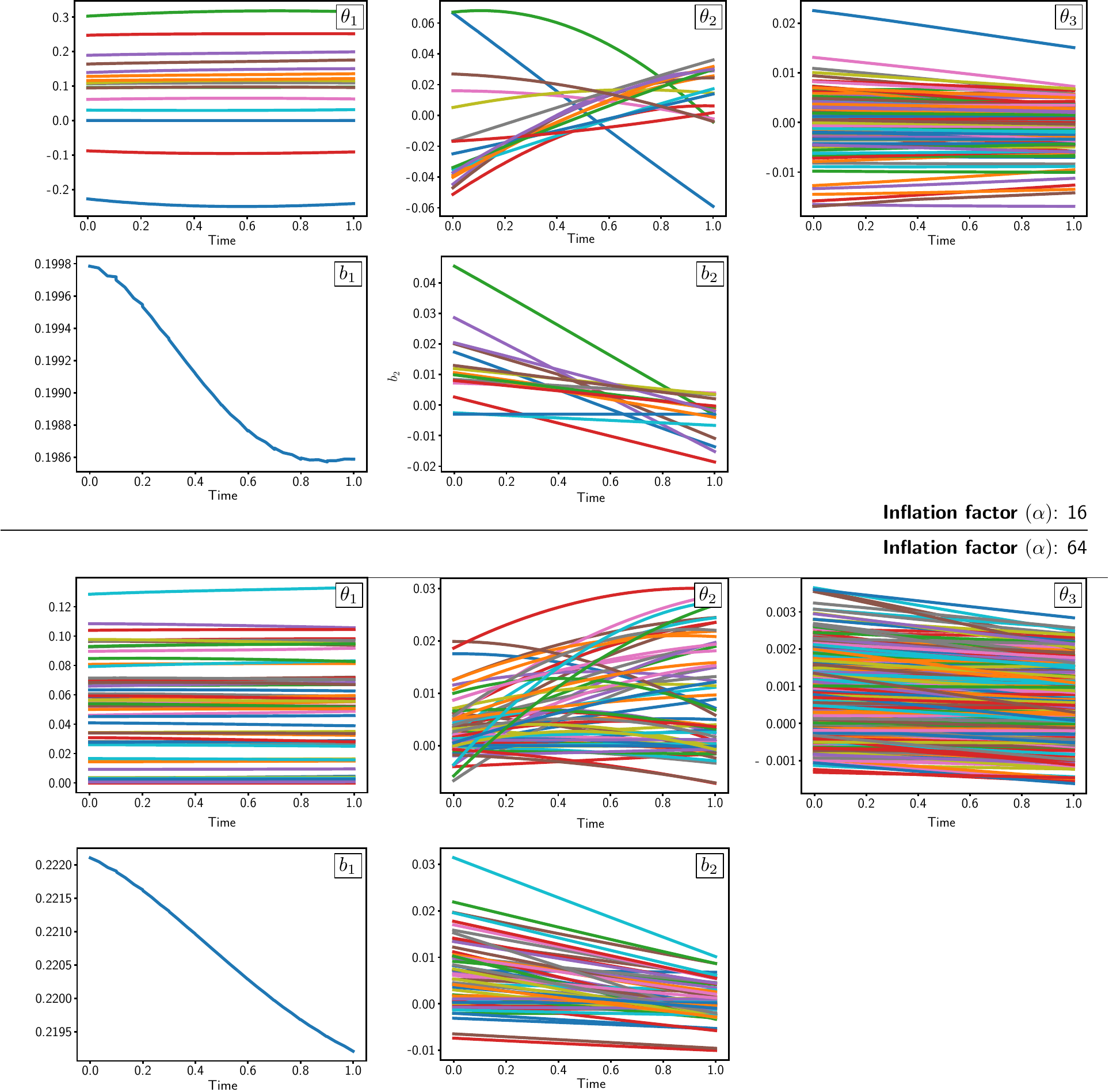}

\caption{Weight evolution across time (i.e., continuous depth) for 15 particles when fitting the \emph{quadratic-like} function using the \UpDown\ model: $\dot{\x}(t) = \theta_1(t) \sigma(\v(t)) + b_1(t),~
\dot{\v}(t) = \theta_2(t) \x(t) + b_2(t) + \theta_3(t) \sigma(\v(t))$. Results are for the dynamic with particles approach. Top: Inflation factor 16. Bottom: Inflation factor 64. Changes in parameter values can clearly be observed. 
}
\label{fig:weight_evolution_quadratic}
\end{figure}

\textbf{Spiral.} Tab.~\ref{tab:number_of_parameters_spiral} shows the number of parameters in each of the four formulations for the spiral experiment. This table complements the Table in Fig.~\ref{fig:spiral_short_long_range} which only showed the number of parameters when using 15 particles.



\begin{table}[h!]
	\caption{Number of parameters for the spiral experiment.}
	\label{tab:number_of_parameters_spiral}
	\centering
	\vskip1ex
	\begin{small}
	\begin{tabular}{r|c|cccc}
		& & \multicolumn{4}{c}{\textbf{Inflation factor}} \\\hline
		& {\textbf{\#Particles}} & 16 & 32 & 64 & 128 \\ \hline
		\multirow{3}{*}{static/dynamic w/ particles} & 15 & 1,116 & 2,172 & 4,284 & 8,508 \\
		& 25 & 1,796 & 3,492 & 6,884 & 13,668 \\
		& 50 & 3,496 & 6,792 & 13,384 & 26,568\\ \hline
		static direct & n/a & 1,282 & 4,610 & 17,410 & 67,586\\ \hline
		dynamic direct & n/a & 6,026 & 22,282 & 85,514 & 334,858
	\end{tabular}
	\end{small}
\end{table}

\end{document}